\documentclass[10pt,twocolumn,letterpaper]{article}

\usepackage[pagenumbers]{wacv} % To force page numbers, e.g. for an arXiv version

\usepackage[utf8]{inputenc} % allow utf-8 input
\usepackage[T1]{fontenc}    % use 8-bit T1 fonts
\usepackage{url}            % simple URL typesetting
\usepackage{booktabs}       % professional-quality tables
\usepackage{amsfonts}       % blackboard math symbols
\usepackage{nicefrac}       % compact symbols for 1/2, etc.
\usepackage{microtype}      % microtypography
\usepackage{xcolor}         % colors

\usepackage{times}
\usepackage{soul}
\usepackage{graphicx}
\usepackage{amsthm,mathrsfs}
\usepackage{verbatim} % for comments
\usepackage{listings} % for comments
\usepackage{algorithm}
\usepackage{algorithmic}
\usepackage{amsmath,amssymb}
\usepackage{textcomp}
\usepackage{multirow}
\usepackage{tikz}
\usetikzlibrary{shapes}
\usepackage{pgfplots}
\pgfplotsset{compat=newest}
\usepackage{subcaption}

\newtheorem{prop}{Proposition}

\newcommand{\iou}{\mathit{IoU}}
\newcommand{\auroc}{\mathit{AUROC}}
\newcommand{\map}{\mathit{mAP}}
\newcommand{\fone}{\mathit{F}_1}
\newcommand{\tpl}{\mathit{TP_l}}
\newcommand{\fpl}{\mathit{FP_l}}
\newcommand{\fnl}{\mathit{FN_l}}

\newcommand{\DKL}{D_\mathrm{KL}}
\newcommand{\CE}{\ell_{\mathrm{CE}}}
\newcommand{\pf}{p_\mathrm{F}}
\newcommand{\scH}{\mathscr{H}}

\DeclareMathOperator*{\argmax}{arg\, max}

\usepackage[pagebackref,breaklinks,colorlinks]{hyperref}

\usepackage[capitalize]{cleveref}
\crefname{section}{Sec.}{Secs.}
\Crefname{section}{Section}{Sections}
\Crefname{table}{Table}{Tables}
\crefname{table}{Tab.}{Tabs.}

\usepackage[normalem]{ulem}

\def\wacvPaperID{2026} % *** Enter the WACV Paper ID here
\def\confName{WACV}
\def\confYear{2024}

\begin{document}

\title{Identifying Label Errors in Object Detection Datasets by Loss Inspection}

\author{Marius Schubert \\
  IZMD, University of Wuppertal\\
  Germany \\
  {\tt\small schubert@math.uni-wuppertal.de}
  \and
  Tobias Riedlinger \\
  IZMD, University of Wuppertal\\
  Germany \\
  {\tt\small riedlinger@math.uni-wuppertal.de}
  \and
  Karsten Kahl \\
  IZMD, University of Wuppertal\\
  Germany \\
  {\tt\small kkahl@math.uni-wuppertal.de}
  \and
  Daniel Kröll \\
  ControlExpert GmbH \\
  Germany \\
  {\tt\small d.kroell@controlexpert.com}
  \and
  Sebastian Schoenen \\
  ControlExpert GmbH \\
  Germany \\
  {\tt\small s.schoenen@controlexpert.com}
  \and
  Siniša Šegvić \\
  University of Zagreb \\
  Croatia \\
  {\tt\small sinisa.segvic@fer.hr}
  \and
  Matthias Rottmann \\
  IZMD, University of Wuppertal\\
  Germany \\
  {\tt\small rottmann@math.uni-wuppertal.de}
}
\maketitle

%%%%%%%%% ABSTRACT
\begin{abstract}
Labeling datasets for supervised object detection is a dull and time-consuming task.
Errors can be easily introduced during annotation and overlooked during review, yielding inaccurate benchmarks and performance degradation of deep neural networks trained on noisy labels.
In this work, we introduce a benchmark for label error detection methods on object detection datasets as well as a theoretically underpinned label error detection method and a number of baselines.
We simulate four different types of randomly introduced label errors on train and test sets of well-labeled object detection datasets. 
For our label error detection method we assume a two-stage object detector to be given and consider the sum of both stages' classification and regression losses.
The losses are computed with respect to the predictions and the noisy labels including simulated label errors, aiming at detecting the latter.
We compare our method to four baselines: a naive one without deep learning, the object detector's score, the entropy of the classification softmax distribution and a probability margin based method from related work.
We outperform all baselines and demonstrate that among the considered methods, ours is the only one that detects label errors of all four types efficiently, which we also derive theoretically.
Furthermore, we detect real label errors a) on commonly used test datasets in object detection and b) on a proprietary dataset. In both cases we achieve low false positives rates, i.e.,\ 
% when considering 200 proposals from our method, 
we detect label errors with a precision for a) of up to 71.5\% and for b) with 97\%.
\end{abstract}

\section{Introduction}
\label{sec: introduction}

\noindent
Nowadays, the predominant paradigm in computer vision is to learn models from data. 
The performance of the model largely depends on the amount of data and its quality, \ie the 
%quality of the input features as well as the labels
diversity of input images and label accuracy~\cite{feng2020deep, hussain2018autonomous, jaeger2020retina, kaur2021survey, kuutti2020survey}. 
Deep neural networks (DNNs) are particularly data hungry~\cite{sun2017revisiting}. 
In this work, we focus on the case of object detection where multiple objects per scene belonging to a fixed set of classes are annotated via bounding boxes~\cite{Everingham10,riedlinger2022activelearning}.

\begin{figure}
    \centering
    \includegraphics[trim={0 0 0 0},clip,width=.4\textwidth]{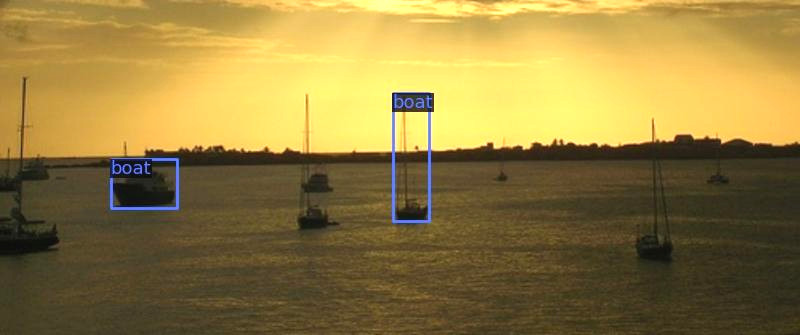}
    \caption{Example image from the Pascal VOC 2007 test dataset with two labeled boats marked by the blue boxes and multiple unlabeled boats.}
    \label{fig: boat}
\end{figure}

%In many applications inputs are processed using deep learning, $\ie$ in safety-critical applications such as automated driving or medical image recognition \cite{feng2020deep, hussain2018autonomous, jaeger2020retina, kaur2021survey, kuutti2020survey}.}
%To obtain a strong predictor, a huge amount of labeled data is needed, e.g. bounding boxes with associated classes selected from a predefined set for the task of object detection. 
In many industrial and scientific applications, the labeling process consists of an iterative cycle of data acquisition, labeling, quality assessment, and model training.
Labeling data is costly, time consuming and error prone, e.g.\ due to inconsistencies caused by multiple human labelers or a change in label policy over time.
Therefore, at least a partial automation of the label process is desirable.
One research direction that aims at this goal is automated label error detection~\cite{northcutt2021confident,rottmann2022automated,dickinson2003detecting}.
%Research directions that aim at this goal are active learning (\cite{settles2009active,brust2018active,roy2018deep,elezi2021not,desai2019adaptive}) and automated label error detection (\cite{northcutt2021confident,rottmann2022automated,dickinson2003detecting}).
%Active learning alternates between data labeling and model training, where the latter is utilized for selecting new images for labeling such that the model's accuracy increases as quickly as possible.
%These methods often assume that the labels are obtained by an error-free oracle, which typically does not hold in real-world applications.
The extent to which noisy labels affect the model performance is studied by~\cite{wu2018soft,buttner2023impact}.
Wu \etal~\cite{wu2018soft} observe that the model is able to tolerate a certain amount of missing annotations in training data without losing too much performance on Pascal VOC and Open Images V3 test sets.
In contrast, Buttner \etal~\cite{buttner2023impact} show that inaccurate labels in terms of annotations size in training data yields to significant decrease of test performance for calculus detection on bitewing radiographs.
Other methods model label uncertainty~\cite{riedlinger2022uncertainty,miller2018dropout} or improve robustness w.r.t.\ noisy labels~\cite{li2020learning,feng2021labels,zhang2019towards,chadwick2019training}.

%Up to now, in contrast to active learning, automated detection of label errors has received less attention. 
Up to now, automated detection of label errors has received less attention. 
There exist some works on image classification datasets~\cite{northcutt2021pervasive,northcutt2021confident,2022multilabel_labelerrors}, one work on semantic segmentation datasets~\cite{rottmann2022automated} and some works for object detection \cite{hu2022probability,koksal2020effect}.
Label errors may affect generalization performance, which makes their detection desirable \cite{northcutt2021pervasive}. 
Furthermore, there is business interest in improving and accelerating the review process by partial automation.

\begin{figure*}
    \centering
    % \begin{subfigure}[c]{0.19\textwidth}
    % \includegraphics[trim={2.8cm 2.0cm 2.8cm 2.0cm},clip,width=\textwidth]{figs/boat_hell.jpg}
    % % \includegraphics[trim={0 0 0 0},clip,width=\textwidth]{figs/drop.jpg}
    % \subcaption{Label Drops}
    % \end{subfigure}
    \begin{subfigure}[c]{0.2\textwidth}
    \includegraphics[trim={0 0.8cm 0 0.8cm},clip,width=\textwidth]{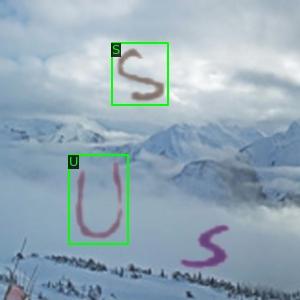}
    \subcaption{Label Drop}
    \end{subfigure}
    \begin{subfigure}[c]{0.2\textwidth}
    \includegraphics[trim={0 0.8cm 0 0.8cm},clip,width=\textwidth]{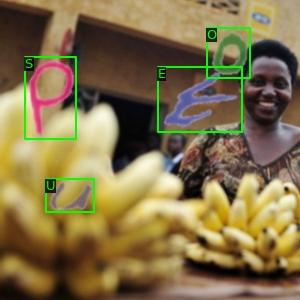}
    \subcaption{Class Flip}
    \end{subfigure}
    \begin{subfigure}[c]{0.2\textwidth}
    \includegraphics[trim={0 0.8cm 0 0.8cm},clip,width=\textwidth]{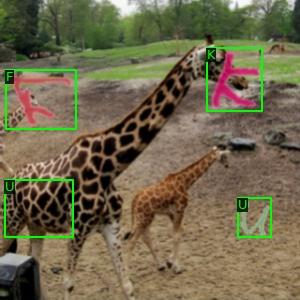}
    \subcaption{Label Spawn}
    \end{subfigure}
    \begin{subfigure}[c]{0.2\textwidth}
    \includegraphics[trim={0 0.8cm 0 0.8cm},clip,width=\textwidth]{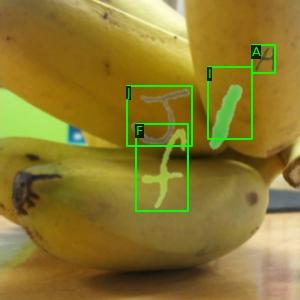}
    \subcaption{Loc. Inaccuracy}
    \end{subfigure}
    \caption{Examples of the different types of simulated label errors. The images are from the EMNIST-Det test dataset~\cite{riedlinger2022activelearning}.}
    \label{fig: label error types}
\end{figure*}

Here, we study the task of label error detection in object detection datasets by a) introducing a benchmark and b) developing a detection method and compare it against four baselines.
We introduce a benchmark by simulating label errors on the BDD100k~\cite{yu2020bdd100k} and EMNIST-Det~\cite{riedlinger2022activelearning} dataset.
The latter is a semi-synthetic dataset consisting of EMNIST letters~\cite{cohen2017emnist} pasted into COCO images~\cite{lin2014microsoft} of which we expect to possess highly accurate labels.
The types of label errors that we consider are missing labels (\emph{drops}), correct localization but wrong classification (\emph{flips}), correct classification but inaccurate localization (\emph{shifts}), and labels that actually represent background (\emph{spawns}).
We address the detection of these errors by a novel method based on monitoring instance-wise object detection loss.
%We then introduce an instance-wise loss method to detect these types of label errors.
% A two-stage detector is used to determine a box-wise loss, which is the sum of the classification and regression loss components of the first and second stage of the prediction and the associated noisy labels including simulated label errors.
We study the effectiveness of our method in comparison to four baselines.
Then, we demonstrate for commonly used object detection test datasets, such as BDD100k~\cite{yu2020bdd100k}, MS-COCO~\cite{lin2014microsoft}, Pascal VOC~\cite{Everingham10} and Kitti~\cite{Geiger2012CVPR}, and also for a proprietary dataset on car part detection %by the company ControlExpert
that our method detects label errors by reviewing moderate sample sizes of 200 images per dataset.
Our contributions can be summarized as follows:
a) we introduce a novel method based on the instance-wise loss for detecting label errors in object detection, b) we introduce a benchmark for identifying four types of label errors on BDD100k and EMNIST-Det, and c) we apply our method to detect label errors in commonly used and proprietary object detection datasets and manually evaluate the error detection performance for moderate sample sizes.

\section{Related Work}
\label{sec: related work}

\noindent
The influence of noisy labels in the training as well as in the test data is an active and current research topic.
The labels for commonly used image classification datasets are noisy~\cite{northcutt2021pervasive} and this also applies to object detection. 
\Cref{fig: boat} shows an image from the Pascal VOC 2007 test dataset containing just two labeled boats, but clearly more can be seen.

For the task of image classification, some learning methods exist that are more robust to label noise~\cite{goldberger2016training,han2018co,hendrycks2018using,jiang2018mentornet,northcutt2021confident,reed2014training,wang2019symmetric,xu2019l_dmi,zhang2017mixup}. 
Also the task of label error detection has been tackled in \cite{chen2019understanding, northcutt2021pervasive} and theoretically underpinned in \cite{northcutt2021confident}.
Chen \etal~\cite{chen2019understanding} filter whole samples with noisy labels but individual label errors are not detected.
Northcutt \etal\ present label errors in image classification datasets and study to which extent they affect benchmark results \cite{northcutt2021pervasive} followed by the introduction of the task of label error detection \cite{northcutt2021confident}.
The latter introduces a confident learning approach, assuming that the label errors are image-independent. 
Then, the joint distribution between the noisy and the true labels with class-agnostic label uncertainties is estimated and utilized to find label errors.
This method allows to find label errors on commonly used image classification (\ie MNIST or ImageNet) and sentiment classification datasets (Amazon Reviews), resulting in improved model performance by re-training on cleaned training data.
This line of works has been recently extended to the task of multi-label classification in \cite{2022multilabel_labelerrors}, where a single object is shown per image but may carry multiple labels.

For object detection, Wu \etal~\cite{wu2018soft} as well as Xu \etal~\cite{xu2019missing} study how noisy training labels affect the model performance, observing that the model is reasonably robust when dropping labels.
% \cite{wu2018soft} show that after dropping $30\%$ of the training labels, the test performance drops by $5$ percent points on the VOC test dataset.
% The same observation applies to object detection, the model is reasonably robust when dropping labels (\cite{xu2019missing}).
To counter label errors in object detection, methods that model label uncertainty~\cite{riedlinger2022uncertainty,miller2018dropout} or more robust object detectors have been developed~\cite{li2020learning,feng2021labels,zhang2019towards,xu2019missing,kang2022finding,buttner2023impact}.
Buttner \etal~\cite{buttner2023impact} simulate label errors and introduce a co-teaching approach for more robust training with noisy training data.
For the task of label error detection, Koksal \etal~\cite{koksal2020effect} simulate different types of label errors in video sequences.
Predictions and labels of consecutive frames are compared and then manually reviewed to eliminate erroneous annotations.
Hu \etal~\cite{hu2022probability} introduce a probability differential method (PD) to identify and exclude annotations with wrong class labels during training.
% Note that, label error detection itself is not evaluated in that work.

For semantic segmentation, a benchmark is introduced by Rottmann and Reese~\cite{rottmann2022automated} to detect missing labels. 
For this purpose, uncertainty estimates are used to predict for each false positive connected component whether a label error is present or not. 
Detection is performed by considering the discrepancy of the given (noisy) label and the corresponding uncertainty estimate.

Our work introduces the first benchmark with four types of label errors for label error detection methods on object detection datasets as well as a label error detection method (that detects all four types of label errors) and a number of baselines.
%The label error detection methods simulate and detect either only class-flips for image classification or missing labels and class-flips for semantic segmentation. 
%The task of object detection is structurally different from the tasks of image classification and semantic segmentation.
%Instead of either the whole image or each pixel being classified, foreground and background must first be distinguished and then only selected parts of the image are classified.
The label error detection methods simulate a) different types of label errors and detect these with the help of a tracking algorithm~\cite{koksal2020effect} for images derived from video sequences or b) class-flips and identify these via a probability differential (PD)~\cite{hu2022probability}, where, however, the focus is on training.
For our benchmark, we randomly simulate four types of label errors and detect them simultaneously with a new and four baseline methods, including PD.
In our method, the discrepancy between the prediction or expectation of the network and the actual labels is used to find label errors. 
This discrepancy is determined by the classification and regression loss from the first and second stage of the detector.
This allows to find not only simulated but also real label errors on commonly used object detection test datasets.

\section{Label Error Detection}
\label{sec: label error detection}
\noindent
In this section we describe our label error benchmark as well as the setup and evaluation for real label errors on commonly used object detection test datasets and a proprietary dataset.
We describe which datasets are used, which types of label errors are considered and the way we simulate label errors inspired by observations that we made in commonly used datasets and by related work.
We then introduce our detection method as well as four additional baseline methods. 
This is complemented with evaluation metrics used to compare the methods with each other on our label error benchmark and the evaluation procedure for commonly used test datasets where we manually review the findings of our method for moderate sample sizes.

\begin{figure*}
    \centering
    \begin{tikzpicture}
        \node at (0,0) {\includegraphics[width=0.825\textwidth]{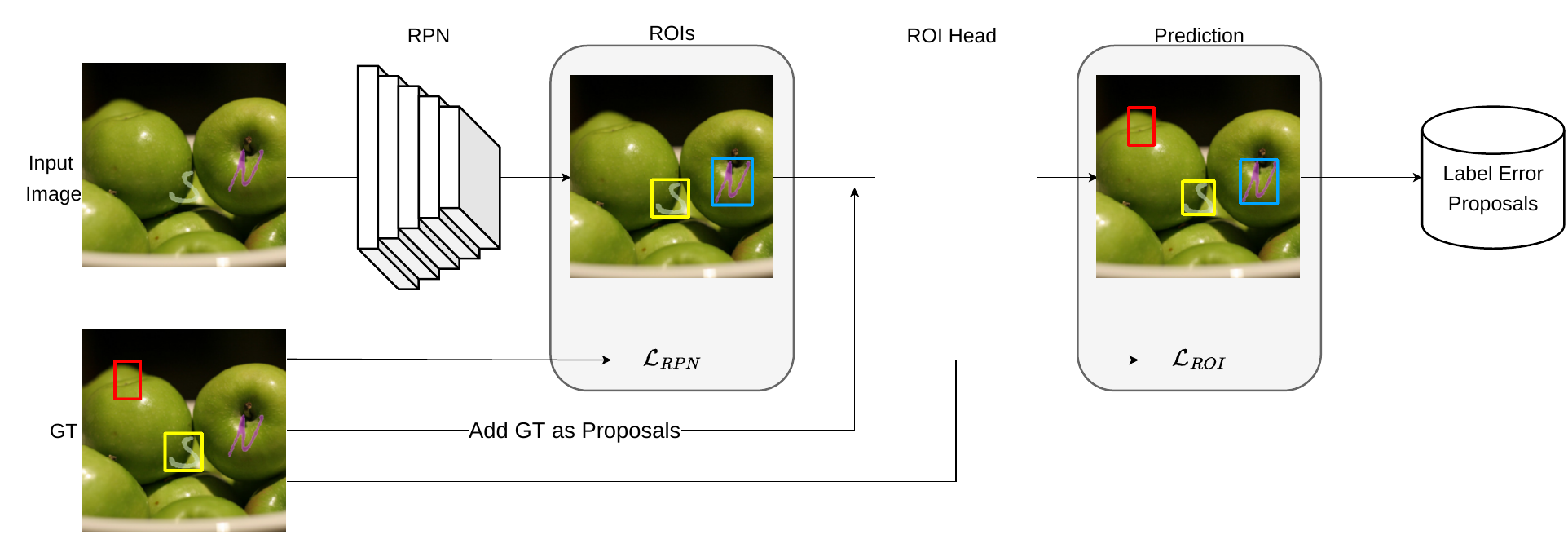}};
        \node at (1.8,0.9) {\includegraphics[width=.09\textwidth]{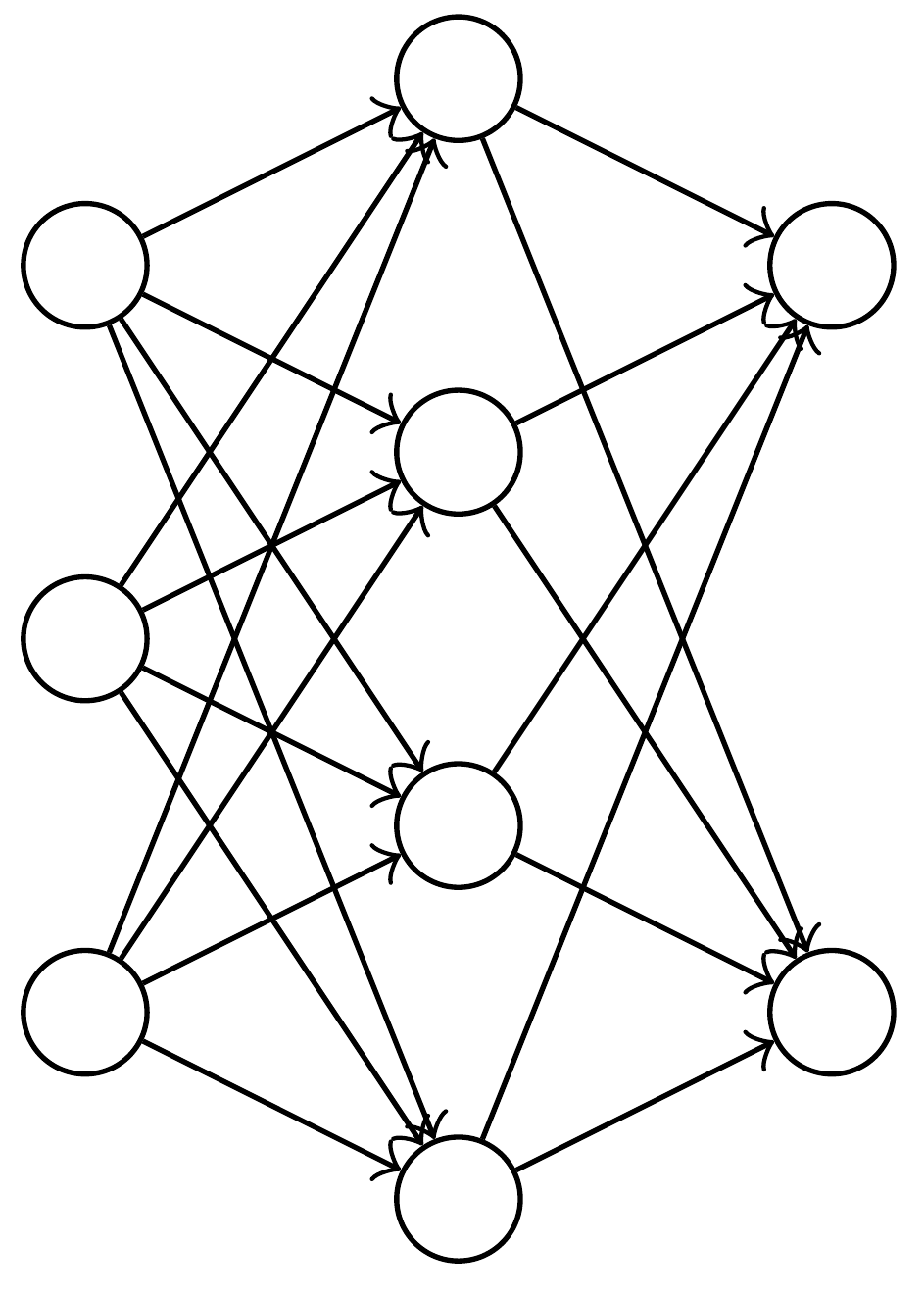}};
    \end{tikzpicture}
    \caption{Visualization of our instance-wise loss method for detecting label errors. The red label indicates a \emph{spawn}, the blue one a \emph{drop} and the yellow one a correct label.}
    \label{fig: loss method}
\end{figure*}

\subsection{Label Error Benchmark}
\label{sec: label error benchmark}
\noindent
In the following, we distinguish between the label error benchmark and the detection of real label errors, where in the former, label errors are simulated (and therefore known) and the performance of the five different methods are evaluated automatically. 
In the latter case, label errors are not simulated but real and therefore automated evaluation is impossible as the real label errors are unknown.
To enable a reliable evaluation, only datasets containing almost no real label errors are used for the benchmark. 
We observe that commonly used datasets in object detection, such as MS-COCO, Pascal VOC or Kitti, contain significant amounts of label errors, thus they are not suitable for the benchmark.
Nevertheless, to demonstrate the performance of our instance-wise loss method on these datasets, a moderate sample size of 200 label error proposals are manually reviewed and counted for each dataset.

\subsubsection{Datasets}
For our benchmark we use the semi-synthetic EMNIST-Det dataset and BDD100k, referred to as BDD. 
EMNIST-Det consists of $20,\!000$ training and $2,\!000$ test images.
To have the best possible labels for BDD, we filter the training and validation split, such that we only use daytime images with clear weather conditions.
This results in $12,\!454$ training images and the validation split is split into equally-sized test and validation sets, each consisting of $882$ images.

\subsubsection{Simulated Label Errors}
We consider four different types of label errors: missing labels (\emph{drops}), correct localization but wrong classification (\emph{flips}), correct classification but inaccurate localization (\emph{shifts}), and labels that actually represent background (\emph{spawns}). 
Any dataset is equipped with a set of $G$ labels, \ie $\mathcal{Y}=\{ b^{(i)}:i=1,...,G \}$ where each label is a tuple $b^{(i)} = (x^{(i)}, y^{(i)}, w^{(i)}, h^{(i)}, c^{(i)} )$ containing the box center $(x^{(i)},y^{(i)})$, the box extent $(w^{(i)},h^{(i)})$ and a class index $c^{(i)}$ from the set of classes $\{1,...,C\}$.
Let $\mathcal{I} = \{ 1,\ldots,G \}$ be the set of indices of all boxes $b^{(i)}\in\mathcal{Y},\, i=1,\ldots,G$.
We now describe all types of label errors applied to $\mathcal{Y}$ and we make the assumption that a single label is only perturbed by one type of label error instead of multiple types.
We choose a parameter $\gamma \in [0,1]$ representing the relative frequency of label errors.

\paragraph{Drops}
For dropping labels, we randomly choose a subset $\mathcal{I}_d$ of $\mathcal{I}$ with cardinality $|\mathcal{I}_d| =  \lfloor \frac{\gamma}{4}\cdot G \rfloor$. We drop all labels $\mathcal{Y}_d = \{ b^{(i)} \, : \, i \in \mathcal{I}_d \}$ and denote $\mathcal{I}_{\setminus d} = \mathcal{I} \setminus \mathcal{I}_d$. Analogously, $\mathcal{Y}_{\setminus d} = \mathcal{Y} \setminus \mathcal{Y}_{d}$.

\paragraph{Flips}
For flipping class labels, we randomly choose a subset $\mathcal{I}_f$ of $\mathcal{I}_{\setminus d}$ with cardinality $|\mathcal{I}_f| =  \lfloor \frac{\gamma}{4}\cdot G \rfloor$ and copy $\tilde{\mathcal{Y}}_f = \mathcal{Y}_f= \{ b^{(i)} \, : \, i \in \mathcal{I}_f \}$. 
Then, we randomly flip the class of every label in $\tilde{\mathcal{Y}_f}$ to a different label.
We denote $\mathcal{I}_{\setminus f} = \mathcal{I}_{\setminus d} \setminus \mathcal{I}_f$ and $\mathcal{Y}_{\setminus f} = (\mathcal{Y}_{\setminus d} \setminus \mathcal{Y}_f) \cup \tilde{\mathcal{Y}}_f$.

\paragraph{Shifts}
To insert \emph{shifts}, we change the localization of labels.
We randomly choose a subset $\mathcal{I}_{sh}$ of $\mathcal{I}_{\setminus f}$ with cardinality $|\mathcal{I}_{sh}| =  \lfloor \frac{\gamma}{4}\cdot G \rfloor$ and copy $\tilde{\mathcal{Y}}_{sh} = \mathcal{Y}_{sh}= \{ b^{(i)} \, : \, i \in \mathcal{I}_{sh} \}$. 
For the shift of a box $\tilde{b}^{(i)}\in\tilde{\mathcal{Y}}_{sh}$, the new values $\tilde{y},\ \tilde{h}$ are determined analogously to $\tilde{x}=\mathcal{N}(x, 0.15\cdot w)$ and $\tilde{w}=\mathcal{N}(w, 0.15\cdot w)$ drawn from a normal distribution with itself as the expected value and $0.15\cdot w$ as the standard deviation. 
% In order not to make the shift too small or too large, we make use of folded normal distributions, such that the intersection over union ($\iou$) of the original label $b^{(i)}\in \mathcal{Y}_{sh}$ and $\tilde{b}^{(i)}\in\tilde{\mathcal{Y}}_{sh}$ is in the interval of $[0.4, 0.7]$, $\forall i=1,...,\lfloor \frac{\gamma}{4}\cdot G\rfloor$.
To avoid the \emph{shift} being too small or too large, the parameters are repeatedly chosen until the intersection over union ($\iou$) of the original label $b^{(i)}\in \mathcal{Y}_{sh}$ and $\tilde{b}^{(i)}\in\tilde{\mathcal{Y}}_{sh}$ is in the interval of $[0.4, 0.7]$, $\forall i=1,...,\lfloor \frac{\gamma}{4}\cdot G\rfloor$.
We denote $\mathcal{I}_{\setminus {sh}} = \mathcal{I}_{\setminus f} \setminus \mathcal{I}_{sh}$ and $\mathcal{Y}_{\setminus sh} = (\mathcal{Y}_{\setminus f} \setminus \mathcal{Y}_{sh}) \cup \tilde{\mathcal{Y}}_{sh}$.

\paragraph{Spawns}
For spawning labels, we randomly choose a subset $\mathcal{I}_{sp}$ of $\mathcal{I}_{\setminus sh}$ with cardinality $|\mathcal{I}_{sp}| =  \lfloor \frac{\gamma}{4}\cdot G \rfloor$ and copy $\tilde{\mathcal{Y}}_{sp} = \mathcal{Y}_{sp}= \{ b^{(i)} \, : \, i \in \mathcal{I}_{sp} \}$. 
Then, we assign every label $\tilde{b}^{(i)}\in\tilde{\mathcal{Y}}_{sp}$ randomly to another image.
Since in our experiments all images in a dataset have the same resolution, this ensures that objects do not appear in unusual positions or outside of an image. 
For instance, a car in BDD is more likely to be found on the bottom part of the image rather than in the sky.
We denote the set of noisy labels as $\tilde{\mathcal{Y}} = \mathcal{Y}_{\setminus sh} \cup \tilde{\mathcal{Y}}_{sp}$.

One example per label error type is shown in \cref{fig: label error types}.
% To weight the label error types equally, we choose a parameter $\gamma$ that represents the noise for each type of label error. 
% We choose randomly four disjoint subsets with length $\gamma\%$ of the total number of labels.
% Then we drop all labels in the first subset, randomly flip the class labels of the second, change the localization for all labels in the third, and let the labels from the fourth subset spawn on randomly picked different images respectively. 
% These four sets of modified labels are disjoint, meaning that a label can at most be affected by one error.
% The selection of which labels are changed and how is done randomly and thus image and label independent.
% Analogously to the definition of the set of original labels $\mathcal{Y}$, the set of noisy labels $\tilde{\mathcal{Y}}$ including all four types of label errors is defined as $\tilde{\mathcal{Y}}=(\mathcal{Y}\setminus(\mathcal{Y}_d\cup \mathcal{Y}_c\cup \mathcal{Y}_{sh}))\cup\tilde{\mathcal{Y}}_c\cup\tilde{\mathcal{Y}}_{sh}\cup\tilde{\mathcal{Y}}_{sp}$ with $\tilde{\mathcal{Y}}=\{ \tilde{b}^{(i)}:i=1,...,G \}$. 
%\MR{\textbf{MR: diese Mengenoperation sieht shcon besser aus. Aber die weiter oben scheinen da nicht mit uebereinzustimmen.}}
Note that the number of labels $G$ is unchanged as the number of \emph{drops} and \emph{spawns} is the same.

\subsection{Baseline Methods}
\noindent
The four baselines that we compare our instance-wise loss method with are based on a) inspecting the labels without the use of deep learning, b) the box-wise detection score c) the classification entropy of the two-stage object detectors and d) the probability differential from \cite{hu2022probability}.
%Note that the tracking label error detection method from \cite{koksal2020effect} is not applicable to our setup.%, since the underlying data of our benchmark is not based on video sequences.

\paragraph{Naive Baseline}
We introduce a naive baseline to show the significant improvement of deep learning in label error detection for object detection over manual label review. 
We assume that all label errors can be smoothly found by taking a single look at all existing noisy labels and the (actually unknown) \emph{drops}, \ie by performing $\lfloor (1+\frac{\gamma}{4} )\cdot G\rfloor$ operations. 
This simplified assumption is of course unrealistic, however the corresponding results can serve as a lower bound for the effort of manual label review.

\paragraph{Detection Score Baseline}
The detection score baseline works as follows:
For a given image from the set of all images of the dataset $z\in Z$, a neural network predicts a fixed number $N_0$ of bounding boxes for the first stage $\hat{\mathcal{B}}_{0,z}=\{(\hat{x}^{(i)},\hat{y}^{(i)},\hat{w}^{(i)},\hat{h}^{(i)},\hat{s}^{(i)}_0): i=1,\ldots,N_0$\}, where $\hat{x}^{(i)},\hat{y}^{(i)},\hat{w}^{(i)},\hat{h}^{(i)}$ represent the localization and $\hat{s}^{(i)}_0\in[0,1]$ the objectness score. 
Then, we add the boxes of the labels as proposals for the second stage to ensure that at least one prediction exists for each label, which is particularly important for the detection of \emph{spawns}.
For this purpose, each ground truth label from $\mathcal{Y}$ is assigned with a detection score of $\hat{s}_0=1$.
After adding the labels to $\hat{\mathcal{B}}_{0,z}$, only those $N_1$ boxes that remain after class-independent non-maximum suppression (NMS) and score thresholding on $\hat{s}_0$ with $s_{\epsilon}\geq 0$, get into the second stage $\hat{\mathcal{B}}_{1,z} = \{(\hat{x}^{(i)},\hat{y}^{(i)},\hat{w}^{(i)},\hat{h}^{(i)},\hat{s}_0^{(i)}):i=1,\ldots,N_1 \}$.
After box refinement and classification as well as NMS on the detection score $\hat{s}_2^{(i)}$, $N_2$ label error proposals remain.
Here, $\hat{s}_2^{(i)}$ is the detection score of the detection head and unlike $\hat{s}_0^{(i)}$, $\hat{s}_2^{(i)}$ represents not only the presence of an object, but also takes the class probabilities of the predicted object into account.
The remaining $N_2$ label error proposals $\hat{\mathcal{B}}_{2,z} = \{(\hat{x}^{(i)},\hat{y}^{(i)},\hat{w}^{(i)},\hat{h}^{(i)},\hat{s}_2^{(i)},\hat{p}_1^{(i)},\ldots,\hat{p}_C^{(i)}) :i=1,\ldots,N_2 \}$, are defined by the localization $(\hat{x}^{(i)},\hat{y}^{(i)},\hat{w}^{(i)},\hat{h}^{(i)})$, the detection score $\hat{s}_2^{(i)}$ and the class probabilities $\hat{p}_1^{(i)},\ldots,\hat{p}_C^{(i)}$.
The predicted class is given by $\hat{c}^{(i)}=\argmax_{k=1,\ldots,C} \hat{p}^{(i)}_k$.
Score thresholding is omitted here, or the score $\tau$ used for this is equal to $0$, since $\tau>10^{-4}$ would suppress most of the label error proposals that detect \emph{spawns}. 
The detection score of these proposals is mostly very close to zero unless a second true label is nearby.
After inferring each image $z\in Z$ as described above, we get label error proposals for the whole dataset by $\bigcup\limits_{z\in Z}\hat{\mathcal{B}}_{2,z}$.

\paragraph{Entropy Baseline}
The entropy baseline follows the same procedure, only the NMS in the first and second stage are based on the respective box-wise entropy rather than the detection score.

\paragraph{Probability Differential Baseline}
For the PD baseline from Hu \etal~\cite{hu2022probability}, we do not add the boxes of the labels as proposals.
Furthermore, score thresholding and NMS is not applied, such that every box $\hat{b}\in\hat{\mathcal{B}}_{0,z}$ also remains in $\hat{\mathcal{B}}_{2,z}$.
After assigning every label with sufficiently overlapping predictions, the probability differential (PD) for every label $b\in\mathcal{Y}$ with class $c$ and the $m$ assigned predictions $\hat{b}^{(i)}$ ($i=1,\ldots,m$) is defined as:
% \begin{align}
%     PD(b) = \frac{\sum\limits_{i=1}^m \big(\hat{p}^{(i)}_\kappa - \max\limits_{c\in\{1,\ldots,C\}\setminus\{\kappa\}} \hat{p}^{(i)}_c\big)}{m}.
% \end{align}
$PD(b) = \frac{\sum\limits_{i=1}^m \big(1 + \max\limits_{k\in\{1,\ldots,C\}\setminus\{c\}} \hat{p}^{(i)}_k  - \hat{p}^{(i)}_c \big)}{2m}$.
The PD of a label is in $[ 0,1 ]$ and intuitively, the more the probabilities of the predictions and the class of the label differ (higher PD) the more likely a label error is present.
Note that \emph{drops} are always overlooked.

\subsection{Instance-wise Loss Method}
\noindent
Our method to detect the introduced label error types (\cref{sec: label error benchmark}) is based on an instance-wise loss for two-stage object detectors. 
The NMS is no longer based on the detection score or the entropy, but on the box-wise loss of the respective stage.
Every prediction $\hat{b}_0\in \hat{\mathcal{B}}_{0,z}$ is assigned with a region proposal loss ($\mathcal{L}_{\mathit{RPN}}$), which is the sum of a classification (binary cross-entropy) and regression (smooth-L1) loss for the labels and the prediction itself. 
The computation of the loss is identical to the one in training.
Since not all labels are associated with a proposal after the first stage, \ie the model may predict only background near a label, we add the labels themselves to the set of label error proposals. 
%Before adding labels, we assign every label with the loss of the most overlapping predicted box.
% For these labels we obtain a predicted proposal from the second stage with a low detection head's score.
% This results in a high loss, because the model does not expect a predicted box, although it has a significant overlap with a label.}
%To ensure that we assign label error proposals for all noisy labels, we add the labels themselves to the set of proposals for the second stage.
After box refinement and classification, every box $\hat{b}_1\in \hat{\mathcal{B}}_{1,z}$ is assigned with a region of interest loss ($\mathcal{L}_{\mathit{ROI}}$), which is the sum of a classification (cross entropy) and regression (smooth-L1) loss for the labels and the prediction itself. 
Then $\mathcal{L}_{\mathit{RPN}}$ and $\mathcal{L}_{\mathit{ROI}}$ are summed up to obtain an instance-wise loss score.
A sketch of our method is shown in \cref{fig: loss method}. 
We can find the \emph{dropped} blue label for ``N'' since the predictions near the object should have a high detection score, resulting in a high first stage classification loss. 
%Therefore, the classification loss should be high in the first stage.
The \emph{spawned} red label is assigned with a high classification loss from the first and second stage, since the assigned predictions should have a score close to zero in the first stage and an almost uniform class distribution in the second stage. 
%In the first stage the predictions near the object have a detection score close to zero, even though there is an object, and in the second stage an almost uniform class distribution is predicted.
Whether the yellow label is a \emph{flip} is irrelevant for the first stage, since the loss should be small either way. 
If the box is classified correctly according to the associated label, there is a large classification loss for a \emph{flip} and a small one otherwise.
The \emph{shifts} are addressed by the first and second stage regression loss.

The intuition behind our method is that a sufficiently well-specified and fitted model has small expected loss on data sampled during training.
Sufficient data sampling and moderate label error rates lead to label errors giving rise to outlier losses which are identified as proposals.
We show that our method separates correct from incorrect labels for a classification model $\widehat{p}$ trained with the cross entropy loss \(\CE\).
\begin{prop}[Statistical Separation of the Cross Entropy Loss]
    Let training and testing labels be given under a stochastic flip in \(p(\cdot | x)\) with probability \(p_\mathrm{F}\).
    A correct label \(y = f(x)\) is given by a true labeling function \(f\) and has probability \(p(f(x)|x) = 1 - \pf\).
    Incorrect labels \(\widetilde{y} \neq f(x)\) are drawn with probability \(p(\widetilde{y} | x) = \pf / (C - 1)\).
    % as above \MR{``as above'' stimmt hier nicht}, l
    Let the label distribution \(p(\cdot|x)\) be PAC-learnable by the hypothesis space of \(\widehat{p}(\cdot | x)\) w.r.t.\ \(\DKL\) (to precision \(\varepsilon\) and confidence \(1 - \delta\)) and let \(\kappa > 0\).
    If \(\pf < \tfrac{C - 1}{C} (1 - 2\kappa)\), we obtain strict separation of the loss function
    % any drawn sample \((x, y)\) involving a true label \(y\) will have a smaller loss than a sample with a false label 
    \begin{align}
        \begin{split}
            \CE(\widehat{p}(x) \| f(x)) &< - \log(1 - \pf - \kappa) \\ 
            &< - \log(\kappa + \tfrac{\pf}{C - 1}) < \CE(\widehat{p}(x) \| \widetilde{y})
        \end{split}
    \end{align}
    % Then, we have for the cross entropy loss
    % \begin{itemize}
    %     \item in the case of the correct label \(y\): \(\CE(\widehat{p}(x) \| y) \leq - \log(1 - \pf - \kappa)\).
    %     \item in the case of an incorrect label \(y\): \(\CE(\widehat{p}(x) \| y) > - \log(\kappa + \tfrac{\pf}{C - 1})\).
    % \end{itemize}
    for any incorrect label \(\widetilde{y} \neq f(x)\) with probability \(1 - \delta\) over chosen training data and with probability \(1 - \tfrac{2 \varepsilon}{\kappa^2}\) over the choice of \(x\).
\end{prop}
We include a proof of this statement in the supplementary material.
The PAC-learnability assumption~\cite{shalev2014understanding} yields rigorous bounds for the deviation of the model \(\widehat{p}\) from the label distribution \(p\) which contains label flips.
Conditioned to the events of drawing correct versus drawing incorrect labels, these bounds carry over to the cross entropy.
These bounds separate the two events with certain probability given in the statement above.

\begin{figure*}
    \centering
    \begin{subfigure}[c]{0.495\textwidth}
    \resizebox{\linewidth}{!}{
    \includegraphics[trim={0.35cm 0 0.6cm 0},clip,height=3.1cm]{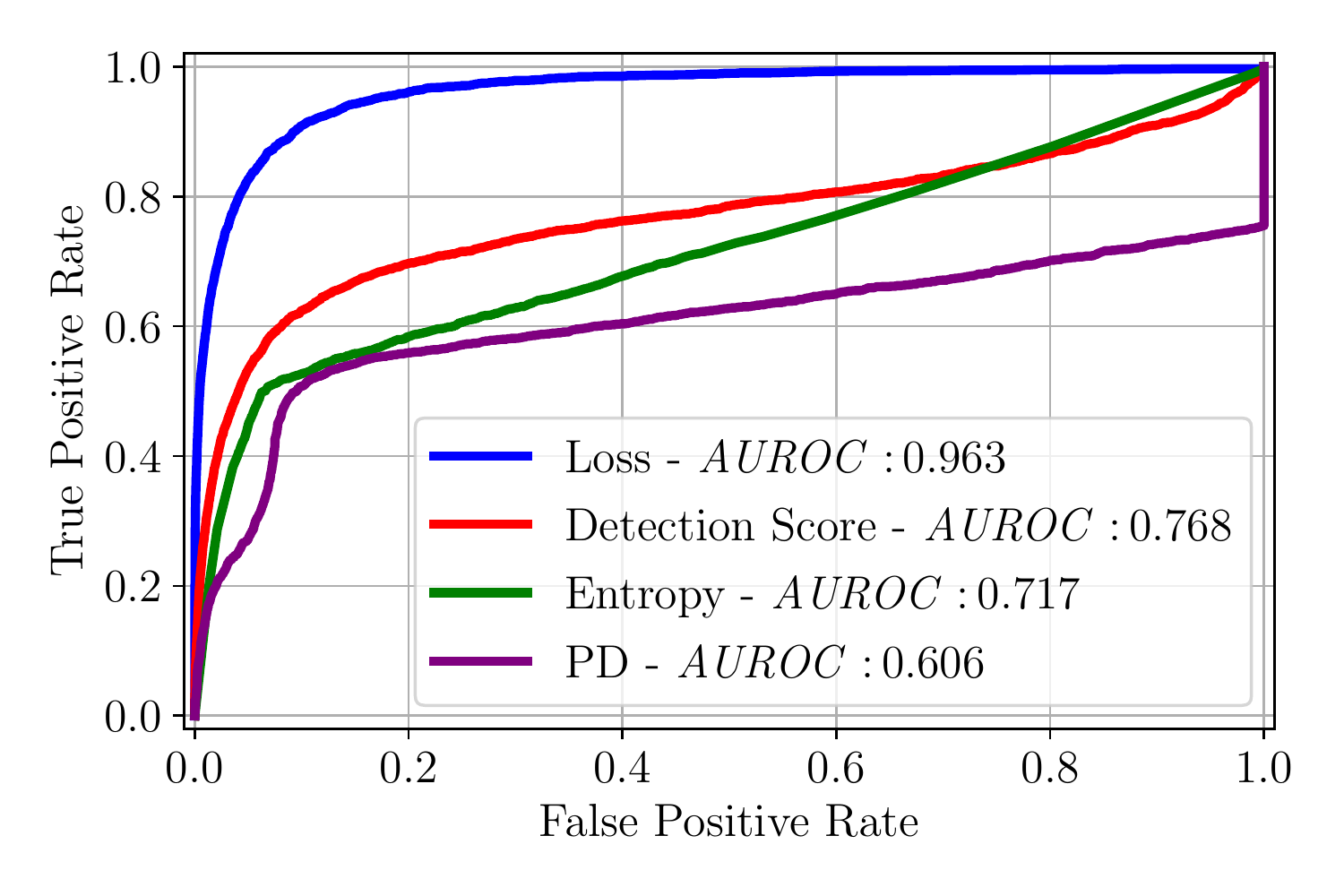}
    \includegraphics[trim={0.6cm 0 0.62cm 0},clip,height=3.1cm]{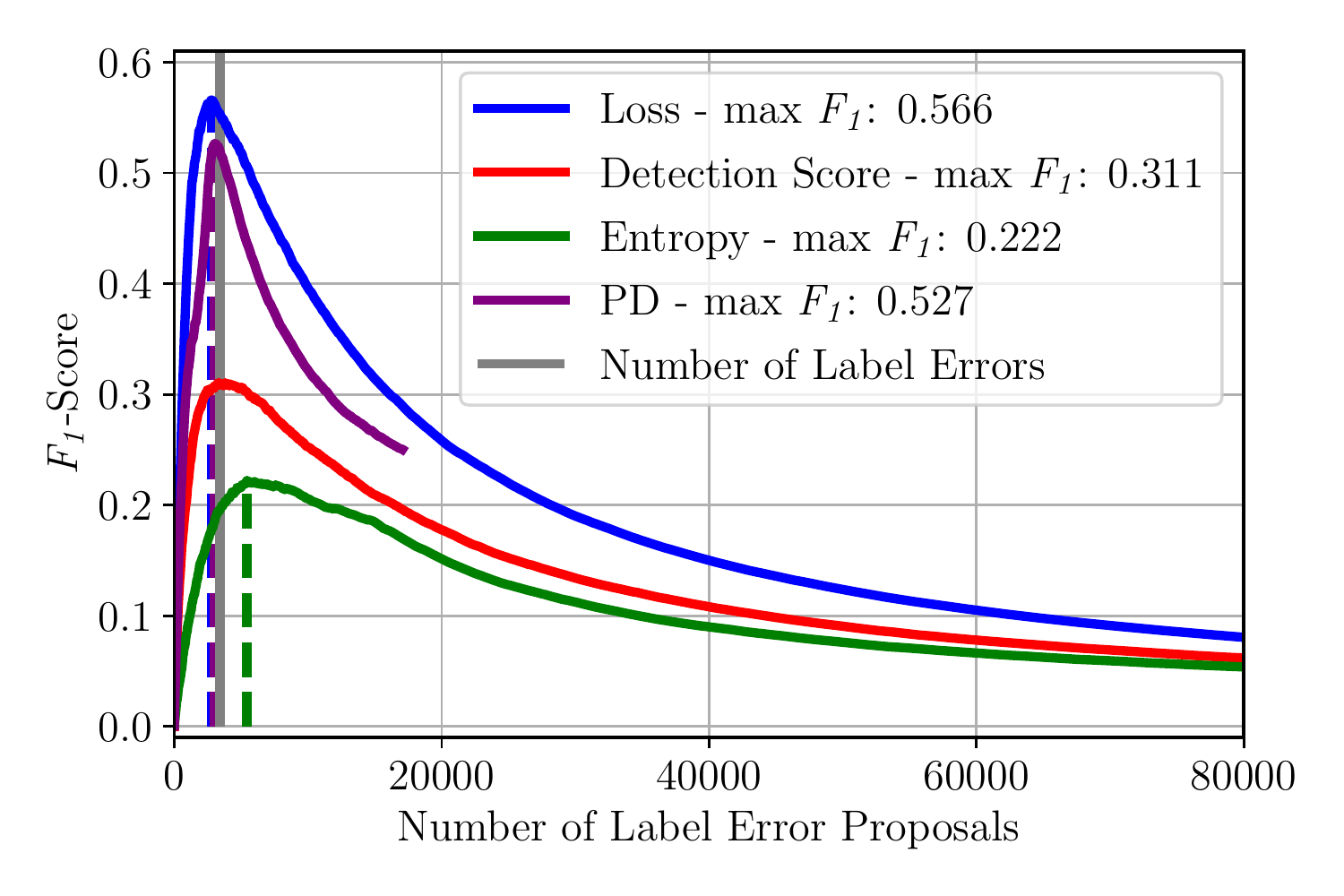}
    }
    \subcaption{Original training data}
    \end{subfigure}
    \begin{subfigure}[c]{0.495\textwidth}
    \resizebox{\linewidth}{!}{
    \includegraphics[trim={0.35cm 0 0.6cm 0},clip,height=3.1cm]{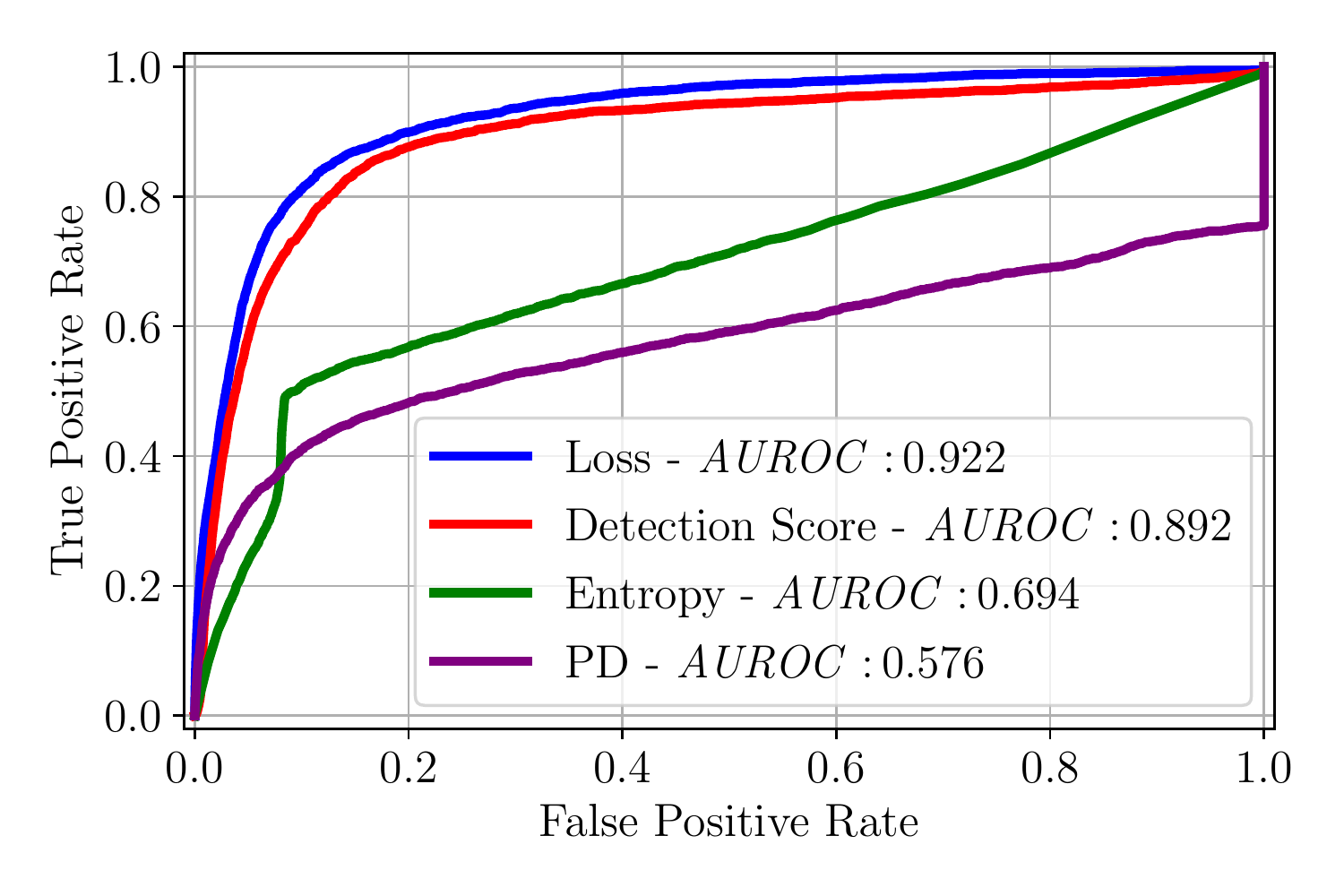}
    \includegraphics[trim={0.6cm 0 0.62cm 0},clip,height=3.1cm]{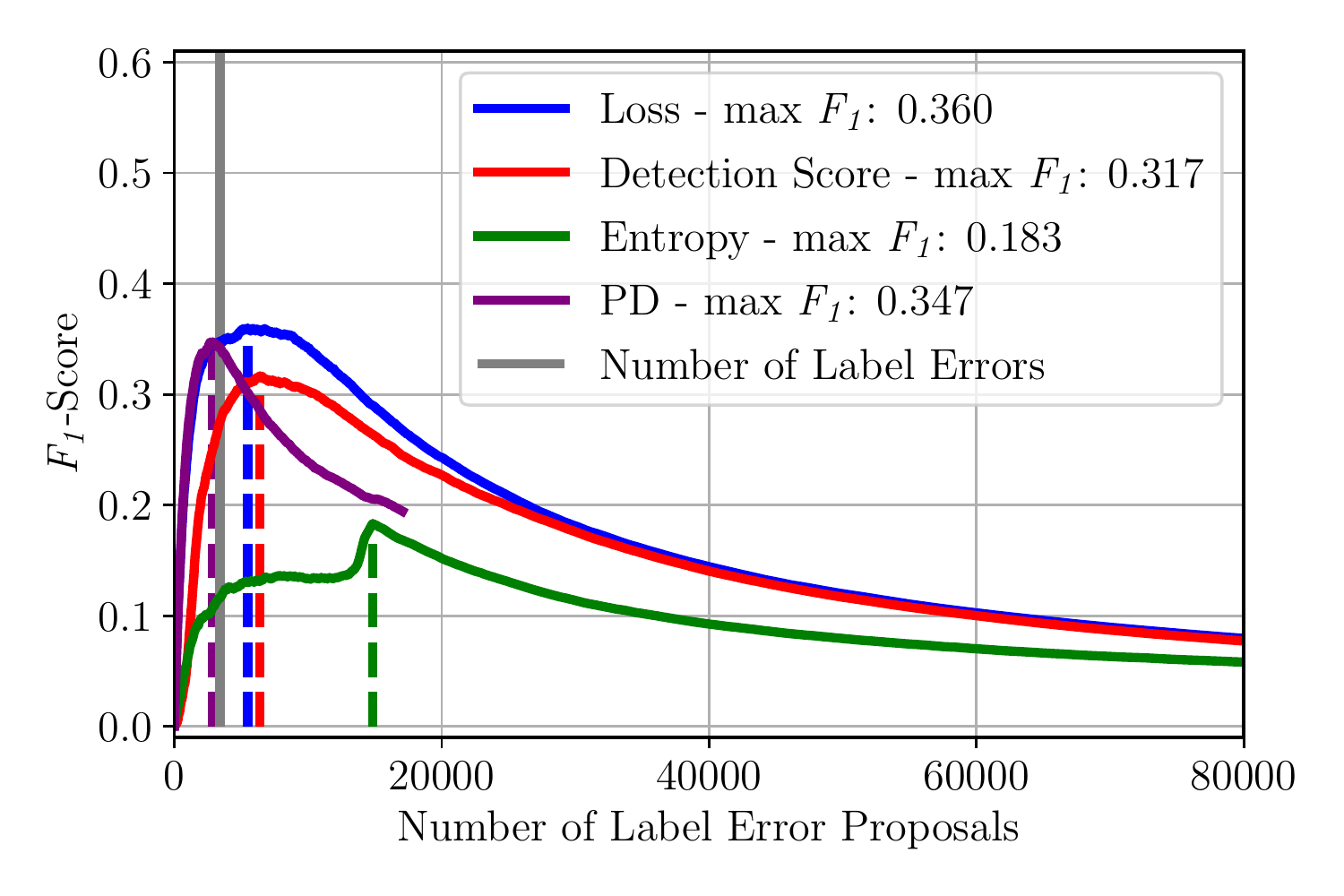}
    }
    \subcaption{Noisy training data with $\gamma=0.2$}
    \end{subfigure}
    \caption{The two left plots in (a) show evaluations based on the predictions of a model trained on original training data and the two right ones in (b) based on noisy training data with $\gamma=0.2$. The number of considered label error proposals depends on threshold $\tau$.}
    \label{fig: auroc-f1}
\end{figure*}

\subsection{Evaluation Metrics}
\noindent
Ignoring that natural label errors exist in EMNIST-Det and BDD, we benchmark the five methods introduced in above by means of our label error simulation.
To this end, we take the label error proposals of the respective method and the set of original labels $\mathcal{Y}$ and decide for every proposal whether it is a label error, which corresponds to a true positive ($\tpl$), or no label error, which corresponds to a false positive ($\fpl$).
Label errors that are not detected are called false negatives ($\fnl$).
A proposal of a label error detector is a $\tpl$ if the $\iou$ between the proposal under consideration and a noisy label on the image is greater or equal to a threshold $1\geq\alpha>0$.
Here, the noisy label categorizes what type of label error is detected by the proposal.
If the $\iou$ is less than $\alpha$, the proposal is a $\fpl$.
After determining this for each proposal from the dataset, the area under the receiver operator characteristic curve ($\auroc$, see~\cite{DBLP:conf/icml/DavisG06}) and $\fone$ values, which is the harmonic mean of precision and recall (see~\cite{dice1945measures}), is calculated according to the decision between $\tpl$ and $\fpl$.
$\fone$ values are determined with thresholding on the score of the respective method (loss/detection score/entropy/PD).
We always choose the optimal threshold, \ie the threshold at which the $\fone$ value is maximized ($\max\,\fone$).
Note, since the naive baseline considers images and thus label error proposals in random order, the associated $\auroc$ values are always $0.5$.

% \begin{figure*}
%     \centering
%     \begin{subfigure}[c]{0.99\textwidth}
%     \includegraphics[height=5.5cm]{figs/bdd_clean/auroc_bdd.pdf}.
%     \includegraphics[height=5.5cm]{figs/bdd_clean/f1_bdd.pdf}
%     \subcaption{Original training data}
%     \end{subfigure}
    
%     \begin{subfigure}[c]{0.99\textwidth}
%     \includegraphics[height=5.5cm]{figs/bdd_pert/auroc_bdd.pdf}
%     \includegraphics[height=5.5cm]{figs/bdd_pert/f1_bdd.pdf}
%     \subcaption{Noisy training data with $\gamma=5$}
%     \end{subfigure}
%     \caption{$\auroc$ curves and values as well as $\fone$ curves and maximal score for different decision thresholds $\tau$ for Swin-T and BDD. The score is naturally given between $0$ and $1$. For the loss and entropy, the respective values of the dataframe were divided by the maximum value. The upper images show evaluations based on the predictions of a model trained on original training data and the lower ones based on noisy training data with $\gamma=5$.}
%     \label{fig: auroc-f1}
% \end{figure*}

\subsection{Detection of Real Label Errors}
\label{sec: detection of real label errors}
\noindent
For commonly used datasets we proceed as follows. 
We consider for each dataset $200$ proposals of our method with highest loss and manually flag them as $\tpl$ or $\fpl$, based on the label policy corresponding to the given dataset. 
Note that we can still compute precision values but we are not able to determine $\auroc$ or $\max\,\fone$ values as the number of total label errors is unknown.
Since several label errors can be detected with one proposal, precision describes the ratio of proposals with at least one label error and the total number of proposals considered, i.e.\ $200$.

\section{Numerical Results}
\label{sec: results}
\noindent
In this section we study label error detection performance on our label error benchmark as well as for real label errors in BDD, VOC, MS-COCO (COCO), Kitti and the %ControlExpert 
\begin{table}
    \centering
    \resizebox{.675\columnwidth}{!}{
    \begin{tabular}{cccc} 
     Dataset & Backbone & $\map_{50}$ & $\map_{50}^{(*)}$ \\
     \toprule
     EMNIST-Det & Swin-T & $98.2$ & $98.0$ \\
     EMNIST-Det & ResNeSt101 & $96.4$ & $95.2$ \\
     BDD & Swin-T & $52.1$ & $50.3$ \\
     BDD & ResNeSt101 & $56.8$ & $52.9$ \\
     \midrule
     COCO & Swin-T & $54.1$ & \\
     Kitti & Swin-T & $38.6$ & \\
     VOC & Swin-T & $83.3$ & \\
     CE & Swin-T & $70.0$ & \\
     \bottomrule
    \end{tabular}
    }
    \caption{Validation of object detection performance on our datasets. $^{(*)}$ indicates learning with simulated label errors ($\gamma=0.2$).}
    \label{tab: map-table}
\end{table}
proprietary dataset (CE).
The benchmark results are presented in terms of $\auroc$ and $\max\,\fone$ values for the joint evaluation of all label error types, \ie when all label error types are present simultaneously, in \cref{sec: benchmark results}.
For the latter, we show how many real label errors we can detect among the top-$200$ proposals for each real-world dataset in \cref{sec: real label error}.

\begin{table*}
    \centering
    
    \resizebox{.8\textwidth}{!}{
    % \begin{tabular}{c|c|c||c|c|c|c||c|c|c|c}
    \begin{tabular}{ccc | cccc | cccc}\toprule
     % \multicolumn{3}{c||}{} & \multicolumn{4}{c||}{$\auroc$} & \multicolumn{4}{c}{$\max\,\fone$} \\
     \multicolumn{3}{c|}{} & \multicolumn{4}{c|}{$\auroc$} & \multicolumn{4}{c}{$\max\,\fone$} \\
     Dataset & Backbone & Train Labels & Loss & Detection Score & Entropy & PD & Loss & Detection Score & Entropy & PD \\
     \toprule
     EMNIST-Det & Swin-T & Original & \textbf{99.46} & \underline{73.24} & 71.49 & 59.67 & \textbf{95.54} & \underline{64.74} & 49.58 & 62.32 \\
     EMNIST-Det & Swin-T & Noisy & \textbf{99.40} & \underline{82.44} & 77.32 & 62.26 & \textbf{93.43} & \underline{62.37} & 45.25 & 62.24 \\
     EMNIST-Det & ResNeSt101 & Original & \textbf{99.84} & \underline{88.45} & 86.70 & 60.59 & \textbf{94.31} & \underline{62.56} & 38.81 & 60.82 \\
     EMNIST-Det & ResNeSt101 & Noisy & \textbf{99.87} & \underline{93.11} & 86.40 & 61.82 & \textbf{90.74} & \underline{59.50} & 34.53 & 59.01 \\
     \midrule
     BDD & Swin-T & Original & \textbf{96.30} & \underline{76.82} & 71.73 & 60.59 & \textbf{56.59} & 31.14 & 22.21 & \underline{52.66} \\
     BDD & Swin-T & Noisy & \textbf{92.16} & \underline{89.21} & 69.42 & 57.58 & \textbf{35.97} & 31.68 & 18.33 & \underline{34.72} \\
     BDD & ResNeSt101 & Original & \textbf{95.79} & \underline{87.47} & 83.58 & 60.31 & \textbf{54.62} & 31.99 & 20.37 & \underline{47.16} \\
     BDD & ResNeSt101 & Noisy & \textbf{92.97} & \underline{90.76} & 78.18 & 56.79 & \textbf{27.85} & 25.65 & 18.10 & \underline{27.74} \\
     \bottomrule
    \end{tabular}
    }
    \caption{Label error detection experiments with two different backbones; higher values are better. Bold numbers indicate the highest $\auroc$ or $\max\,\fone$ per experiment and underlined numbers are the second highest.}
    \label{tab: auroc-f1-scores}
\end{table*}

\subsection{Implementation Details}
\label{sec: implementation details}
\noindent
We implemented our benchmark and methods in the open source MMDetection toolbox~\cite{mmdetection}.
Our models are based on a Swin-T transformer and a ResNeSt101 backbone, both with a CascadeRoIHead as the object detection head, with a total number of trainable parameters of approx. $72$M and $95$M. 
%Dataset dependent hyperparameters for training are provided in the appendix. 
As hyperparameters for the label error benchmark we choose relative frequency of label errors $\gamma=0.2$, the value for score thresholding after the first stage $s_{\epsilon}=0.25$, the value for score thresholding after the second stage $\tau=0$ and the $\iou$-value $\alpha=0.3$ from which a proposal for a label error is considered a $\tpl$. 
We show performance results for the respective models and for each dataset in \cref{tab: map-table}.
The upper half shows results on original ($\map_{50}$) and noisy training data ($\map_{50}^*$), for which $\gamma=0.2$ also holds.
With sufficient training data and a moderate label error rate ($\gamma=0.2$), the models still generalize well, resulting in  $\map$ values comparable to models trained without simulated label errors. 
Thus, to make benchmarks evaluations trustworthy, whether already published or still under development, in particular the underlying test datasets should contain as few label errors as possible.
The bottom half of \cref{tab: map-table} presents the performance of the models that we use for predicting label errors on real datasets. 
All models have been trained and evaluated on the original datasets (without any label modification).

\paragraph{Datasets}
For the detection of real label errors we use the same split for BDD as introduced in \cref{sec: label error benchmark} as well as VOC, COCO, Kitti and CE. %, in the following named as VOC, COCO and Kitti.
The training data for VOC consists of ``2007 train'' + ``2012 trainval'' and we predict label errors on the ``2007 test''-split.
COCO is trained on the train split and label errors are predicted on the validation split from 2017.
For Kitti we use a scene-wise split, resulting in $5$ scenes ($S=\{2,8,10,13,17\}$) and $1,\!402$ images for evaluation as well as $16$ scenes ($\{0,1,\ldots,20\}\setminus S$) and $6,\!407$ images for training.
The subset of CE data used includes $20,\!100$ images for training and $1,\!070$ images for evaluation. 
In the images, a car is in focus and the task is to do a car part detection.
The labels consist of $29$ different classes and divide the car into different parts, \ie the four wheels, doors, number plate, mirrors, bumper, etc.
Compared to the static academic datasets, the CE dataset is dynamic and thus of heterogeneous quality.

\begin{table}
    \centering
    \resizebox{.805\linewidth}{!}{
    \begin{tabular}{l|c c|c c c c} 
    Dataset & Label Error & Prec. & Spawn & Drop & Flip & Shift \\
     \midrule
     BDD & 34 & 15.5 & 3 & 2 & 26 & 3  \\ % 169 wrong
     \midrule
     Kitti & 96 & 47.5 & 75 & 0 & 4 & 17 \\ % 105 wrong
     \midrule
     COCO & 50 & 24.5 & 14 & 1 & 18 & 17 \\ % 151 wrong
     COCO$^{(*)}$ & 125 & 61.0 & 0 & 125 & 0 & 0 \\ % 78 wrong proposals
     \midrule
     VOC & 23 & 11.5 & 13 & 0 & 10 & 0 \\ % 177 wrong
     VOC$^{(*)}$ & 175 & 71.5 & 0 & 175 & 0 & 0 \\ % 57 wrong proposals
     \midrule
     CE$^{(*)}$ & 194 & 97.0 & 0 & 0 & 0 & 0 \\
     \bottomrule
    \end{tabular}
    }
    \caption{Categorization of the top-$200$ proposals for real label errors with the loss method for the Swin-T backbone. $^{(*)}$ indicates the evaluation of proposals based on the detection of \emph{drops}.}
    \label{tab: top-200}
\end{table}

\subsection{Benchmark Results for Simulated Label Errors}
\label{sec: benchmark results}
\noindent
\begin{figure*}
    \centering
    \includegraphics[width=.825\textwidth]{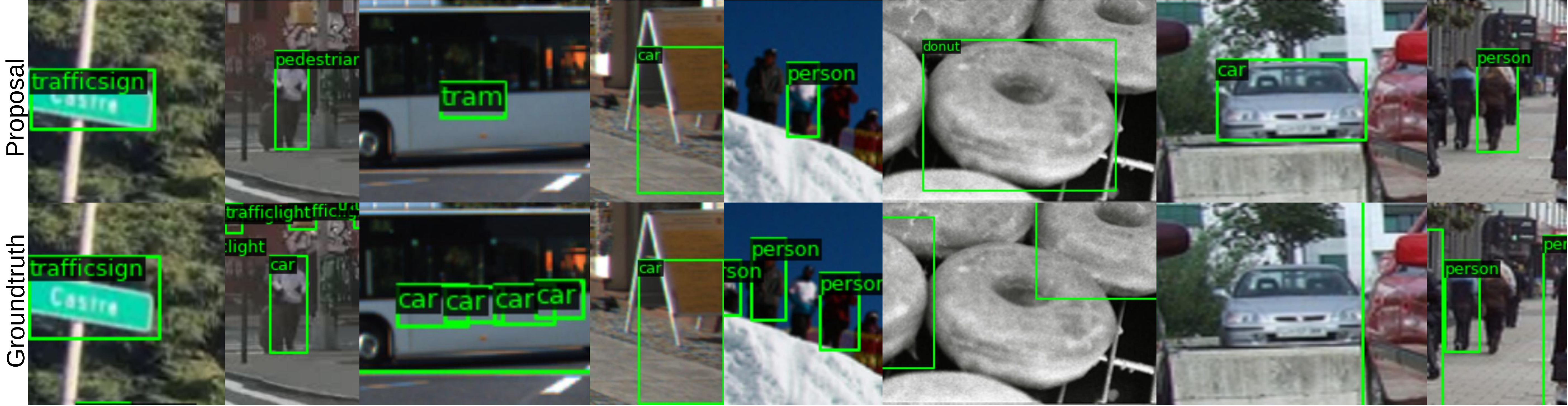}
    \caption{Visualization of detected label errors in real test datasets. The top row of images depicts the label error proposals and the bottom row the corresponding labels from the dataset. The image pairs belong from left to right in steps of two to BDD, Kitti, COCO and VOC.}
    \label{fig: label error collage}
\end{figure*}
\Cref{tab: map-table} shows that although $20\%$ of the training labels are modified, the performance in terms of $\map_{50}$ to $\map_{50}^*$ only decreases by a maximum of $1.2$ percent points (pp) for EMNIST-Det and $3.9$ pp for BDD. 
In both cases, the performance decreases more for the backbone containing more trainable parameters (ResNeSt101). 
This is consistent with the results for image classification from Northcutt \etal~\cite{northcutt2021pervasive}. 
Architectures with fewer trainable parameters seem more suitable for handling label errors in the training data, possibly due to the network having less capacity to overfit the label errors.
\Cref{fig: auroc-f1} shows exemplary plots for $\auroc$ and $\fone$ curves for the Swin-T backbone and BDD.
On the two left plots we show results based on original training data and the two right plots based on noisy training data.
The ranking of the methods is not identical everywhere: in terms of $\auroc$, loss (our method) is superior, followed by detection score, then entropy (our baselines) and finally PD.
In terms of $\max\,\fone$, PD outperforms the detection score and the entropy but is inferior compared to the loss.
Because $\auroc$ considers rates and ($\max$) $\fone$ considers absolute values and the number of label error proposals varies widely (PD = number of labels $G$, here $17,\!064$; others > $80,\!000$), the methods behave very differently with respect to $\auroc$ and $\max\,\fone$.
However, our loss method outperforms all other methods on both metrics.
% Remarkably, loss/entropy with noisy training data show reduced performance, namely $4.1/2.3$ pp $\auroc$ and $20.6/3.9$ pp $\max\,\fone$.  
% The detection score, on the other hand, has an almost identical $\max\,\fone$ value with only $0.6$ pp difference, whereas the $\auroc$ even increases by $12.4$ pp to $89.2\%$. Nevertheless, the loss outperforms the detection score by $3.0$ pp in $\auroc$ and by $4.3$ pp in $\max\,\fone$.
Note, that the small step in the upper right of each of the $\auroc$ plots are the false negatives according to the label errors ($\fnl$), \ie the simulated label errors that are not found by the methods. 
This number of $\fnl$ is vanishingly small in relation to all simulated label errors, with the exception of PD as the method is not able to detect \emph{drops}.
The generally observed behavior for BDD also does not change when looking at the results for the ResNeSt101 backbone in \cref{tab: auroc-f1-scores}.
When comparing the results for the different backbones with each other the $\auroc$ for the loss and PD seems to remain similar, whereas the $\auroc$ for detection score/entropy increases by $10.65/11.85$ pp for original training data and $1.55/8.76$ pp for noisy training data.
The situation is different for the $\max\,\fone$ values.
For label error detection, loss/entropy/PD performs superior with the Swin-T backbone for original training data ($1.97/1.84/5.50$ pp).
In particular, the loss and PD seem to handle the noisy training data more effectively, resulting in $8.12$ pp $\max\,\fone$ difference between Swin-T and ResNeSt for the loss and $6.98$ pp difference for PD.
The detection score increases by $0.85$ pp with the ResNeSt101 backbone on original training data, but on noisy data the Swin-T outperforms the ResNeSt101 by $6.03$ pp.
Also for EMNIST-Det it holds that the loss outperforms the detection score and both outperform the entropy.
In contrast to the results of BDD, the detection score slightly outperforms PD in all EMNIST-Det experiments also in terms of $\max\,\fone$.
The $\auroc$ for loss appears to be stable across backbone and training data quality with only a maximum $0.47$ pp difference overall. 
% The detection score and entropy are superior in terms of $\auroc$ with the ResNeSt101 backbone, but worse in terms of $\max\,\fone$.
The $\auroc$ values for the detection score and entropy are superior with the ResNeSt101 backbone, but inferior in terms of $\max\,\fone$ and the detection score performs superior in terms of $\auroc$ based on noisy training data, but inferior in terms of $\max\,\fone$.
For PD, the $\auroc$ seems to be rather stable comparing the two backbones, but the $\max\,\fone$ is superior for Swin-T.

%All results are based on the simultaneous prevalence of all four label error types. 
%Further $\auroc$ and $\max\,\fone$ results for the evaluation of individual label error types are presented in the appendix.
\subsection{Evaluation for Real Label Errors}
\label{sec: real label error}
\noindent
We now aim at detecting real instead of simulated label errors.
The considered real-world datasets apart from BDD (VOC, COCO, Kitti, CE) are more similar in complexity to BDD than to EMNIST-Det. 
For BDD we observed in \cref{sec: benchmark results} that the loss method for the Swin-T backbone seems to be more stable according to label errors in the training data, as especially the $\max\,\fone$ values for the loss and noisy labels are superior for Swin-T than for ResNeSt101.
As we suspect label errors in the VOC, COCO, Kitti and CE training datasets, we use the Swin-T backbone to detect as many label errors as possible.
Furthermore, we showed in \cref{tab: auroc-f1-scores} that the loss method outperforms the detection score, entropy and PD in each presented experiment,
%of our presented backbone-dataset-training data quality combination, 
hence we detect label errors using only the loss method in the following.
Since we manually look at all proposals individually and we are not able to look at all proposals (\ie about $265,\!000$ for VOC), we categorize the top-$200$ proposals into $\tpl$ or $\fpl$. 
If a $\tpl$ is found we also note which type of label error is present and if we are not sure whether the proposal is $\tpl$ or $\fpl$, we conservatively interpret it as $\fpl$.
The results are summarized in \cref{tab: top-200}.
For BDD, there are at least $34$ label errors, which mostly consist of \emph{flips}.
Since Kitti consists of image sequences, it happens that one label error appears on several consecutive frames. 
When this happens, it usually affects objects that are visible on previous frames but are covered by, for instance, a bus for several frames but are still labeled.
Label error proposals that fall into ``Don't Care'' areas are not considered.
In total, we find $96$ label errors with a precision of $47.5\%$ on Kitti. 
%, where $75$ \emph{spawns} are responsible for the majority of label errors.
% In Kitti there are some labels that are marked as ``Don't Care''; label error proposals that fall into these areas are not considered.
% In total, we find $96$ label errors with a precision of $47.5\%$ on Kitti, where $75$ \emph{spawns} are responsible for the majority of label errors.
% Unlike BDD and Kitti, which have images of street scenes, COCO and VOC consist of images of different everyday scenes that really differ from image to image.
% As even the labels of the same classes are very different from each other \MR{\textbf{MR: this statement I do not understand}}, we have some false negatives in the test dataset, \ie labels that the model would predict as background in a normal test setting. 
% However, in our method we add the labels before the second stage to our set of label error proposals. \MR{\textbf{MR: Also here I am lost.}}
% Since the model is forced to classify especially the boxes of labels for the actual false negatives, an almost uniform class distribution is predicted.
% The resulting loss is so high that these proposals end up in the reviewed top-$200$ proposals.
As COCO and VOC consist of images of different everyday scenes that really differ from image to image, the variability of the representation of objects is very high in these two datasets.
Since a label error proposal is enforced for each label, this also applies to the labels that are classified as background.
In a usual test setting, these labels would have been false negatives of the model, \ie overlooked labels.
The resulting loss is so high that these proposals end up in the reviewed top-$200$ proposals.
Nevertheless, $50$ label errors can be detected on COCO and $23$ on VOC. 
When dealing with these two datasets, we noticed that \emph{drops} are the most present label error type, although we did not find any among the top-$200$ proposals. 
We use this knowledge to restrict the proposals to those that have a class-independent $\iou$ with the labels of the image of less than $\alpha$.
Using this subset and re-reviewing the top-$200$ proposals, we are able to find $125$ \emph{drops} with a precision of $61.0\%$ for COCO and $175$ \emph{drops} with a precision of $71.5\%$ for VOC. 
For the calculation of the precision see \cref{sec: detection of real label errors}.
Prior knowledge about the label quality of the dataset and the types of label errors that occur helps to detect a specific type of label error. 
From the high precisions for VOC and COCO, we conclude that our method can help to correct the label errors resulting in cleaner benchmarks.
Exemplary label errors for the above datasets are shown in \cref{fig: label error collage}. 
The first proposal detects a \emph{shift}, the second a \emph{flip}, the third and fourth a \emph{spawn} and the remaining proposals detect \emph{drops}. 
For CE, we filter the proposals by \emph{drops}, resulting in 194 detected \emph{drops} with a precision of $97\%$.
% For CE, we know that labels in particular are missing. 
% Thus, we filter the proposals by \emph{drops}. 
% In total we detect 194 \emph{drops} with a precision of $97\%$.
%More examples are shown in the appendix.

\section{Conclusion}
\label{sec: conclusion}
\noindent
In this work, we introduced a benchmark for label error detection for object detection datasets.
% We for the first time simulated and evaluated four different types of label errors on two selected datasets that appear to be suitable for further method development.
% Furthermore, we developed a novel method based on instance-wise loss scoring and compare it with four baselines.
We for the first time simulated and evaluated four different types of label errors on two selected datasets.
We also developed a novel method based on instance-wise loss scoring and compare it with four baselines.
Our method prevails by a significant margin in experiments on our simulated label error benchmark.
In our experiments with real label errors, we found a number of label errors in prominent datasets as well as in a proprietary production-level dataset. 
% It turns out that some prior knowledge on the occurrence of the different label error types can substantially leverage the performance of our method. \MR{\textbf{MR: prior knowledge is a restricting factor. We had this discussion once. Since we are only considering four types of label errors, we could always consider them separately. Also, from a deeper analysis of the loss distribution for all four label error types, practitioners / reviewers could potentially learn even more and improve their efficiency beyond what we have demonstrated. This could be discussed briefly. Limitations and Foreseeable Impact could be the title of a section in the appendix treating these points more in detail.}}
With the evaluation for individual label error types we can detect real label errors on commonly used test datasets in object detection with a precision of up to $71.5\%$.
Furthermore, we presented additional findings. 
Models with less parameters are more robust to label errors in training sets while models with more parameters suffer more. 
%In addition, in a label error type specific study we found that \MS{the instance-wise loss outperforms the score and entropy baseline in most cases according to $\auroc$ and $\max\,\fone$. In the rare cases where instance-wise loss is not the best method, the difference from the best is only marginal.}
% We make our code for benchmark, evaluation and method publicly available at \textit{GitHub}.
We make our code publicly available at \textit{GitHub}.
% We make our code for benchmark, evaluation and method publicly available at \href{https://github.com/schubertm/identifying_label_errors_in_od}{\textit{GitHub}}.

% \section*{Ethical Statement}
% There are no ethical issues.
\clearpage
\newpage

\section*{Acknowledgments}
\noindent
We gratefully acknowledge financial support by the state Ministry of Economy, Innovation and Energy of Northrhine Westphalia (MWIDE) and the European Fund for Regional Development via the FIS.NRW project BIT-KI, grant no.\ EFRE-0400216, by the German Federal Ministry for Economic Affairs and Climate Action within the project “KI Delta Learning“ grant no.\ 19A19013Q and by the German Federal Ministry for Education and Research within the project "UnrEAL" grant no.\ 01IS22069.
We also gratefully acknowledge the \href{www.gauss-centre.eu}{Gauss Centre for Supercomputing e.V.} for funding this project by providing computing time through the John von Neumann Institute for Computing on the GCS Supercomputer JUWELS at Jülich Supercomputing Centre.
M. S., T. R. and M. R. acknowledge helpful discussion with H. Gottschalk about the analysis of the presented method.

%%%%%%%%% REFERENCES
{\small
\bibliographystyle{ieee_fullname}
\bibliography{biblio}
}

\clearpage
\newpage

\clearpage
\begin{figure*}[h]
    \centering
    \includegraphics[width=\textwidth]{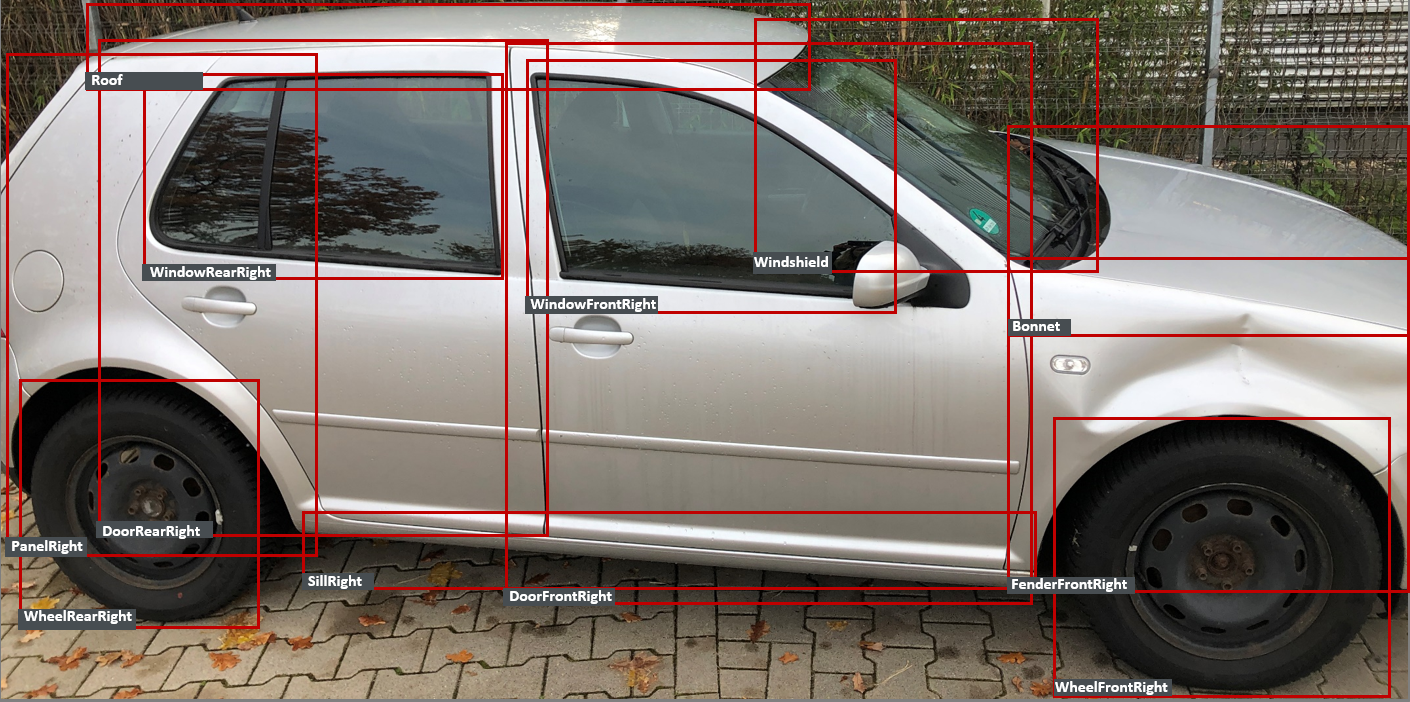}
    \caption{Example image from the CE test data with labels and a missing ``MirrorRight''.}
    \label{fig: ce car}
\end{figure*}
\begin{table*}[h]
    \centering    
    \resizebox{.75\textwidth}{!}{
    % \begin{tabular}{l | c c c c}
    \begin{tabular}{l cccc}
    \toprule
    Dataset & Batch Size & Image Resolution & \# Training Iterations & Learning Rate \\
    \toprule
    EMNIST-Det & $24$ & $300\times 300$ & $24,\!000/48,\!000^{(*)}$ & $0.02$ \\
    BDD & $4$ & $1333\times 800$ & $150,\!000/250,\!000^{(*)}$ & $0.01$ \\
    \midrule
    Kitti & $6$ & $1000\times 600$ & $70,\!000$ & $0.01$ \\
    COCO & $12$ & $1000\times 600$ & $250,\!000$ & $0.02$ \\
    VOC & $6$ & $1000\times 600$ & $70,\!000$ & $0.02$ \\
    CE & $4$ & $1000\times 600$ & $200,\!000$ & $0.01$ \\
    \bottomrule
    \end{tabular}
    }
    \caption{Training hyperparameters for the Swin-T and the ResNeSt101 ($^{(*)}$) backbone.}
    \label{tab: training-hyperparameter}
\end{table*}

\section*{Supplementary Material}
\paragraph{Proprietary Dataset}
An exemplary test image including labels for the proprietary dataset (CE) is shown in \cref{fig: ce car}. 
The labels divide the car into parts, such as the two wheels ``WheelFrontRight'' and ``WheelRearRight'' as well as doors, roof, etc. 
The example also includes a \emph{drop} with the missing mirror ``MirrorRight''. 

\paragraph{Dataset Dependent Parameters for Training}
The dataset-dependent hyperparameters for training are stated in \cref{tab: training-hyperparameter}.
The original images from EMNIST-Det have an image resolution of $320\times 320$ pixels, i.e., we do not artificially scale them to a higher image resolution. 
The BDD images also contain many small labels while having a high original resolution ($1280\times 720$), which is a challenging setup. 
To get the best possible label error detection, we keep this high resolution and rescale the images to $1333\times 800$ pixels. 
Kitti, COCO, VOC and CE are each rescaled to an image resolution of $1000\times 600$ pixels. 
The batch size for all datasets is in the range of $4$-$24$, the initial learning rate is either $0.02$ or $0.01$ depending on the dataset, and the number of training iterations is in the range of $24,\!000$-$250,\!000$.
All numbers apply to the Swin-T backbone except the numbers $^{(*)}$ for the training iterations of EMNIST-Det and BDD, which apply to the ResNeSt101 backbone. 
All other hyperparameters are identical for the different architectures. 
The files for the configurations used in training, also containing the precise values of the above hyperparameters, are published with the code on \textit{GitHub}.
% The files for the configurations used in training are published with the code on \href{https://github.com/schubertm/identifying_label_errors_in_od}{\textit{GitHub}}.

\paragraph{Benchmark Results for Individual Simulated Label Error Types}
In our experiments, all label errors occur simultaneously, but the evaluation can also be conditioned on the individual label error types.
For \emph{drops} or \emph{flips} we consider only the false positives according to $\tilde{y}$, \ie all boxes that have a maximum class-wise $\iou$ of less than $\alpha$($=0.3$) with all noisy labels of the associated image.
Then, we can calculate $\auroc$ and $\max\,\fone$ values on this subset.
We do the same for the \emph{shifts}, except that we only consider the true positives according to $\tilde{\mathcal{Y}}$. 
For the \emph{spawns}, we must consider both true positives and false positives according to $\tilde{\mathcal{Y}}$, since the predicted class, that overlaps sufficiently with the \emph{spawned} label, can be the same as the class of the \emph{spawn} itself.

The benchmark results for individual simulated label errors are stated in \cref{tab: auroc-f1-scores-appendix}.
For \emph{drops}, the detection score and instance-wise loss perform similarly well, with the $\auroc$ values differing by at most $1.92$ pp and a minimal $\auroc$ of $91.72\%$.
The difference in the $\max\,\fone$ values is more pronounced, with the loss at EMNIST-Det outperforming the detection score by $3.03$ to $6.56$ pp.
For BDD, the detection score of Swin-T is superior to the loss by up to $9$ pp, whereas the loss for ResNeSt101 outperforms the detection score by up to $10.22$ pp.
The entropy reaches a maximum of $88.22\%/56.70\%$ $\auroc/\max\,\fone$ for EMNIST-Det and $73.14\%/7.11\%$ for BDD, which is far from the numbers achieved for the loss and the detection score.
PD is not able to detect \emph{drops}, as the bounding boxes of the labels are also the label error proposals itself. 

A similar behavior can be observed for the \emph{flips}, where the $\auroc$ values for loss and detection score only differ by a maximum of $1.64$ pp. 
In terms of $\max\,\fone$ the loss outperforms the detection score and entropy in every case.
PD performs inferior in terms of $\auroc$ compared to the loss, but in terms of $\max\,\fone$ PD outperforms the loss for BDD based on both backbones trained on noisy data.

For the \emph{shifts}, the detection score and PD have similar performance as the naive baseline in terms of $\auroc$ and all $\max\,\fone$ values are $<14\%$. 
% For the shifts, the $\auroc$ values are even close to $50\%$, so that the detection score outperforms the naive baseline ($50\%\auroc$) only marginally.
Except for BDD trained on noisy data, where entropy performs superior to the loss, loss outperforms all baselines.

For the \emph{spawns}, the detection score performs similar compared to the \emph{shifts}.
PD performs well especially in terms of $\max\,\fone$, where PD even outperforms the loss for RestNeSt101 on BDD with noisy training data by $20.16$ pp, otherwise loss is superior to PD.
In the cases where entropy outperforms loss, the difference is at most $5.44$ pp in terms of $\auroc$ and $3.69$ pp in terms of $\max\,\fone$.

The detection score can neither reliably detect the \emph{shifts} nor the \emph{spawns}, whereas the entropy cannot detect the \emph{drops} and \emph{flips} well, especially for complicated problems such as BDD.
PD cannot reliably detect the \emph{shifts} and is not able to detect \emph{drops} by design.
All in all, the loss method is the only one of those presented that can detect all four different types of label errors efficiently. 

\begin{table*}
    \centering    
    \resizebox{.985\textwidth}{!}{
    % \begin{tabular}{c||c|c|c||c|c|c|c||c|c|c|c}
    %  & \multicolumn{3}{c||}{} & \multicolumn{4}{c||}{$\auroc$} & \multicolumn{4}{c}{$\max\,\fone$} \\
    \begin{tabular}{cccc | cccc | cccc}
    \toprule
     & \multicolumn{3}{c|}{} & \multicolumn{4}{c|}{$\auroc$} & \multicolumn{4}{c}{$\max\,\fone$} \\
    Label Error Type & Dataset & Backbone & Train Labels & Loss & Score & Entropy & PD & Loss & Score & Entropy & PD \\
     \toprule
     \multirow{8}{*}{Drop} & EMNIST-Det & Swin-T & Original & \underline{98.94} & \textbf{99.12} & 88.16 & 0.00 & \textbf{94.91} & \underline{89.63} & 56.70 & 0.00 \\
      & EMNIST-Det & Swin-T & Noisy & \underline{98.85} & \textbf{99.19} & 88.22 & 0.00 & \textbf{93.27} & \underline{90.24} & 48.97 & 0.00 \\
      & EMNIST-Det & ResNeSt101 & Original & \textbf{99.66} & \underline{99.65} & 78.33 & 0.00 & \textbf{93.58} & \underline{87.02} & 32.82 & 0.00 \\
      & EMNIST-Det & ResNeSt101 & Noisy & \underline{99.91} & \textbf{99.94} & 78.79 & 0.00 & \textbf{86.42} & \underline{81.03} & 19.21 & 0.00 \\
      \cmidrule{2-12} 
      & BDD & Swin-T & Original & \underline{94.92} & \textbf{96.05} & 51.48 & 0.00 & \underline{41.80} & \textbf{48.38} & 2.37 & 0.00 \\
      & BDD & Swin-T & Noisy & \underline{91.72} & \textbf{93.64} & 52.88 & 0.00 & \underline{37.93} & \textbf{46.93} & 1.45 & 0.00 \\
      & BDD & ResNeSt101 & Original & \textbf{94.52} & \underline{93.61} & 73.14 & 0.00 & \textbf{45.89} & \underline{35.67} & 7.11 & 0.00 \\
      & BDD & ResNeSt101 & Noisy & \textbf{91.89} & \underline{91.84} & 62.25 & 0.00 & \textbf{26.29} & \underline{22.62} & 1.75 & 0.00 \\
     \midrule
     \multirow{8}{*}{Flip} & EMNIST-Det & Swin-T & Original & \underline{99.74} & \textbf{99.78} & 91.09 & 99.34 & \textbf{92.89} & \underline{90.08} & 59.51 & 86.79 \\
      & EMNIST-Det & Swin-T & Noisy & \underline{99.62} & \textbf{99.83} & 90.79 & 99.51 & \textbf{89.42} & \underline{88.70} & 49.32 & 87.44 \\
      & EMNIST-Det & ResNeSt101 & Original & \underline{99.96} & \textbf{99.97} & 78.95 & 99.07 & \textbf{90.77} & \underline{86.70} & 31.65 & 82.93 \\
      & EMNIST-Det & ResNeSt101 & Noisy & \underline{99.89} & \textbf{99.94} & 78.50 & 98.83 & \textbf{81.49} & 80.35 & 18.98 & \underline{80.49} \\
      \cmidrule{2-12} 
      & BDD & Swin-T & Original & \textbf{99.68} & 98.36 & 50.63 & \underline{98.53} & \textbf{74.54} & 58.79 & 2.75 & \underline{73.86} \\
      & BDD & Swin-T & Noisy & \textbf{99.56} & 98.12 & 50.06 & \underline{98.32} & \underline{60.31} & 58.91 & 2.13 & \textbf{71.23} \\
      & BDD & ResNeSt101 & Original & \textbf{99.80} & \underline{98.16} & 75.93 & 97.96 & \textbf{72.81} & 54.38 & 7.12 & \underline{69.95} \\
      & BDD & ResNeSt101 & Noisy & \textbf{99.31} & \underline{97.24} & 64.34 & 97.13 & \underline{44.94} & 40.15 & 2.18 & \textbf{61.75} \\
     \midrule
     \multirow{8}{*}{Shift} & EMNIST-Det & Swin-T & Original & \textbf{99.80} & 51.52 & \underline{93.55} & 40.71 & \textbf{91.76} & 11.14 & \underline{49.41} & 10.61 \\
      & EMNIST-Det & Swin-T & Noisy & \textbf{99.56} & 50.26 & \underline{88.01} & 50.70 & \textbf{87.86} & 10.92 & \underline{40.71} & 10.88 \\
      & EMNIST-Det & ResNeSt101 & Original & \textbf{99.67} & 51.28 & \underline{86.14} & 45.54 & \textbf{88.65} & 11.28 & \underline{30.32} & 10.56 \\
      & EMNIST-Det & ResNeSt101 & Noisy & \textbf{99.30} & 53.73 & \underline{80.52} & 51.91 & \textbf{85.97} & 13.99 & \underline{25.65} & 10.95 \\
      \cmidrule{2-12} 
      & BDD & Swin-T & Original & \textbf{65.49} & 51.76 & \underline{61.17} & 50.24 & \textbf{16.78} & 11.22 & \underline{14.99} & 10.55 \\
      & BDD & Swin-T & Noisy & \underline{57.23} & 52.57 & \textbf{57.91} & 52.85 & \underline{12.73} & 11.44 & \textbf{12.88} & 10.94 \\
      & BDD & ResNeSt101 & Original & \textbf{65.84} & 51.51 & \underline{63.37} & 54.19 & \textbf{17.56} & 11.40 & \underline{14.76} & 11.86 \\
      & BDD & ResNeSt101 & Noisy & \underline{55.92} & 50.85 & \textbf{56.18} & 52.54 & \textbf{12.58} & 10.87 & \underline{12.17} & 11.13 \\
     \midrule
     \multirow{8}{*}{Spawn} & EMNIST-Det & Swin-T & Original & \textbf{99.37} & 75.62 & \underline{97.92} & 97.04 & \textbf{98.87} & 19.89 & 65.08 & \underline{78.97} \\
      & EMNIST-Det & Swin-T & Noisy & \textbf{99.68} & 50.95 & \underline{98.48} & 97.16 & \textbf{97.77} & 19.26 & 59.18 & \underline{79.33} \\
      & EMNIST-Det & ResNeSt101 & Original & \textbf{99.84} & 57.98 & \underline{99.40} & 96.12 & \textbf{98.06} & 18.96 & 37.84 & \underline{74.19} \\
      & EMNIST-Det & ResNeSt101 & Noisy & \textbf{99.93} & 76.31 & \underline{99.33} & 94.89 & \textbf{94.93} & 15.89 & 35.39 & \underline{67.92} \\
      \cmidrule{2-12} 
      & BDD & Swin-T & Original & \textbf{98.48} & 66.33 & \underline{98.07} & 92.09 & \textbf{74.97} & 2.23 & 20.24 & \underline{50.81} \\
      & BDD & Swin-T & Noisy & \underline{90.55} & 78.13 & \textbf{92.98} & 78.00 & \textbf{17.98} & 9.21 & \underline{11.32} & 10.94 \\
      & BDD & ResNeSt101 & Original & \underline{95.80} & 79.79 & \textbf{97.00} & 87.57 & \textbf{60.38} & 5.04 & 13.75 & \underline{38.71} \\
      & BDD & ResNeSt101 & Noisy & \underline{90.30} & 89.19 & \textbf{95.74} & 76.07 & 7.39 & 6.92 & \underline{11.08} & \textbf{28.55} \\
     \bottomrule
    \end{tabular}
    }
    \caption{$\auroc$ and $\max\,\fone$ values for loss, detection score (Score), entropy and PD for all dataset-backbone-training label combinations; higher values are better. Bold numbers indicate the highest $\auroc$ or $\max\,\fone$ per experiment and underlined numbers are the second highest.}
    \label{tab: auroc-f1-scores-appendix}
\end{table*}

\paragraph{Different Noise Intensity in Training}
\Cref{tab: noise-ablation} shows $\map$, $\auroc$ and $\max\,\fone$ values for different noise intensities for Swin-T on the BDD training dataset. 
In our experiments, it makes no difference whether the labels of the training data contain $5\%$ or $20\%$ noise, the $\map$ is between $50.2\%$ and $50.4\%$, where the model has an $\map$ of $52.1\%$ due to training on the original training data. 
All $\map$ evaluations are based on the test data with original and thus unmodified labels.

On the one hand, the $\auroc$/$\max\,\fone$ values decrease with increasing noise intensity by $3.69$/$20.62$ pp for loss, by $2.31$/$3.88$ pp for entropy and by $3.01$/$17.94$ pp for PD, respectively.
For the detection score, on the other hand, the $\auroc$ value increases by $12.39$ pp from $76.82\%$ to $89.21\%$ and the $\max\,\fone$ value increases only marginally by $0.54$ pp to $31.68\%$. 
Nevertheless, the loss outperforms the detection score/entropy/PD in every case by at least $3.40$/$21.68$/$35.71$ pp in terms of $\auroc$ and by at least $4.29$/$17.64$/$0.38$ pp in terms of $\max\,\fone$.
All $\auroc$/$\max\,\fone$ evaluations are based on the test data with $\gamma=0.2$ and thus on the identical label basis as for \cref{tab: auroc-f1-scores}.

\begin{table*}
    \centering    
    \resizebox{.985\linewidth}{!}{
    % \begin{tabular}{c|c||c||c|c|c|c||c|c|c|c} 
    % \multicolumn{2}{c||}{} &  & \multicolumn{4}{c||}{$\auroc$} & \multicolumn{4}{c}{$\max\,\fone$} \\
    \begin{tabular}{ccc|cccc|cccc} 
    \toprule
    \multicolumn{2}{c}{} &  & \multicolumn{4}{c|}{$\auroc$} & \multicolumn{4}{c}{$\max\,\fone$} \\
     $\gamma$ & \# train images & $\map_{50}$ & Loss & Detection Score & Entropy & PD & Loss & Detection Score & Entropy & PD \\
     \toprule
     0 & 12,454 & $52.1$ & \textbf{96.30} & \underline{76.82} & 71.73 & 60.59 & \textbf{56.59} & 31.14 & 22.21 & \underline{52.66} \\
     0.05 & 12,454 & $50.4$ & \textbf{93.44} & \underline{88.09} & 71.76 & 59.16 & \textbf{43.36} & 30.78 & 18.54 & \underline{42.98} \\
     0.1 & 12,454 & $50.2$ & \textbf{93.21} & \underline{89.05} & 70.98 & 58.56 & \textbf{39.36} & 30.79 & 18.53 & \underline{37.92} \\
     0.2 & 12,454 & $50.3$ & \textbf{92.61} & \underline{89.21} & 69.42 & 57.58 & \textbf{35.97} & 31.68 & 18.33 & \underline{34.72} \\
     \bottomrule
    \end{tabular}
    }
    \caption{Validation of object detection performance and label error detection experiments for different noise for training Swin-T on BDD; higher values are better. Bold numbers indicate the highest $\auroc$ or $\max\,\fone$ per experiment and underlined numbers are the second highest.}
    \label{tab: noise-ablation}
\end{table*}

\begin{table*}
    \centering    
    \resizebox{.985\linewidth}{!}{
    % \begin{tabular}{c|c||c||c|c|c|c||c|c|c|c} 
    % \multicolumn{2}{c||}{} &  & \multicolumn{4}{c||}{$\auroc$} & \multicolumn{4}{c}{$\max\,\fone$} \\
    \begin{tabular}{ccc|cccc|cccc} 
    \toprule
    \multicolumn{2}{c}{} &  & \multicolumn{4}{c|}{$\auroc$} & \multicolumn{4}{c}{$\max\,\fone$} \\
     $\gamma$ & \# train images & $\map_{50}$ & Loss & Detection Score & Entropy & PD & Loss & Detection Score & Entropy & PD \\
     \toprule
     0 & 1,556 & $45.1$ & \textbf{94.79} & 69.56 & \underline{72.61} & 59.86 & \textbf{58.67} & 30.59 & 25.95 & \underline{50.77} \\
     0 & 3,113 & $49.7$ & \textbf{95.25} & \underline{73.69} & 72.70 & 59.93 & \textbf{56.48} & 31.38 & 24.64 & \underline{50.13} \\
     0 & 6,227 & $51.3$ & \textbf{95.18} & \underline{74.95} & 73.21 & 60.12 & \textbf{55.79} & 31.55 & 24.52 & \underline{49.78} \\
     0 & 12,454 & $52.1$ & \textbf{96.30} & \underline{76.82} & 71.73 & 60.59 & \textbf{56.59} & 31.14 & 22.21 & \underline{52.66} \\
     \midrule
     0.2 & 1,556 & $40.7$ & \textbf{94.82} & \underline{84.98} & 73.53 & 58.73 & \textbf{44.47} & 27.07 & 18.72 & \underline{33.87} \\
     0.2 & 3,113 & $46.9$ & \textbf{93.35} & \underline{88.90} & 70.31 & 58.98 & \textbf{35.45} & 28.38 & 18.28 & \underline{33.45} \\
     0.2 & 6,227 & $49.1$ & \textbf{93.08} & \underline{90.31} & 70.80 & 58.37 & \textbf{33.60} & 30.38 & 18.18 & \underline{33.43} \\
     0.2 & 12,454 & $50.3$ & \textbf{92.61} & \underline{89.21} & 69.42 & 57.58 & \textbf{35.97} & 31.68 & 18.33 & \underline{34.72} \\
     \bottomrule
    \end{tabular}
    }
    \caption{Validation of object detection performance and label error detection experiments for different noise and number of images for training Swin-T on BDD; higher values are better. Bold numbers indicate the highest $\auroc$ or $\max\,\fone$ per experiment and underlined numbers are the second highest.}
    \label{tab: image-ablation}
\end{table*}

\paragraph{Different Amount of Training Images}
\Cref{tab: image-ablation} shows $\map$, $\auroc$ and $\max\,\fone$ values for different amounts of training images for Swin-T on BDD.
Therefore, the subsets with fewer images are always included in the subsets with more images and the identically sized subsets with different noise intensities contain the same images.

The $\map$ increases the more images are used for training and the less label errors exist in the training data.
Here, the model trained on $6,\!227$ and unmodified labels ($\gamma=0$) has a $0.8$ points higher $\map$ than the model trained on $12,\!454$ images with $\gamma=0.2$. 
In this case, after comparing the performances, it is worth to review and improve the underlying labels instead of labeling new images and add them to the training set.

The $\auroc$ values increase as the number of images increases with $\gamma=0$.
With $\gamma=0.2$, the values for loss and entropy decrease as the number of images increases. 
The $\max\,\fone$ values decrease independently of $\gamma$ with increasing number of images for loss and entropy, whereas the values increase for detection score. 
The decrease in $\auroc$ and $\max\,\fone$ values for loss and entropy could be due to overfitting of the model.
For PD, $\auroc$ and $\max\,\fone$ values remain almost constant for the respective datasets.
However, the loss always outperforms all baselines in terms of $\auroc$ and $\max\,\fone$.
All $\auroc$/$\max\,\fone$ evaluations are based on the test data with $\gamma=0.2$ and thus on the identical label basis as for \cref{tab: auroc-f1-scores}.

\paragraph{Real Label Errors}
Further detected real label errors are presented in \cref{fig: label-error-collage-appendix}.
The top row shows examples for BDD, where all found label errors are \emph{flips}, except for the third proposal from the right.
This proposals can be interpreted as two label errors.
Either the ``car'' label on the ``bus'' is wrong (\emph{spawn}) and the bus was forgotten to be labeled (\emph{drop}), or the localization is inaccurate (\emph{shift}) and the label has a wrongly assigned class (\emph{flip}).
The middle and bottom rows represent detected real label errors on COCO and VOC. 
All proposals show \emph{drops} and at the fourth proposal from the left in the middle row ``pizza'', the two small labels ``pizza'' are also count as \emph{spawns} resulting in three label errors for this single proposal.
%one can even discuss two \emph{spawns}, since the two smaller labels ``pizza'' should exactly reflect the label detected as \emph{drop}. 

\begin{figure*}
    \centering
    \begin{subfigure}[c]{0.995\textwidth}
    \resizebox{\linewidth}{!}{
    \includegraphics[]{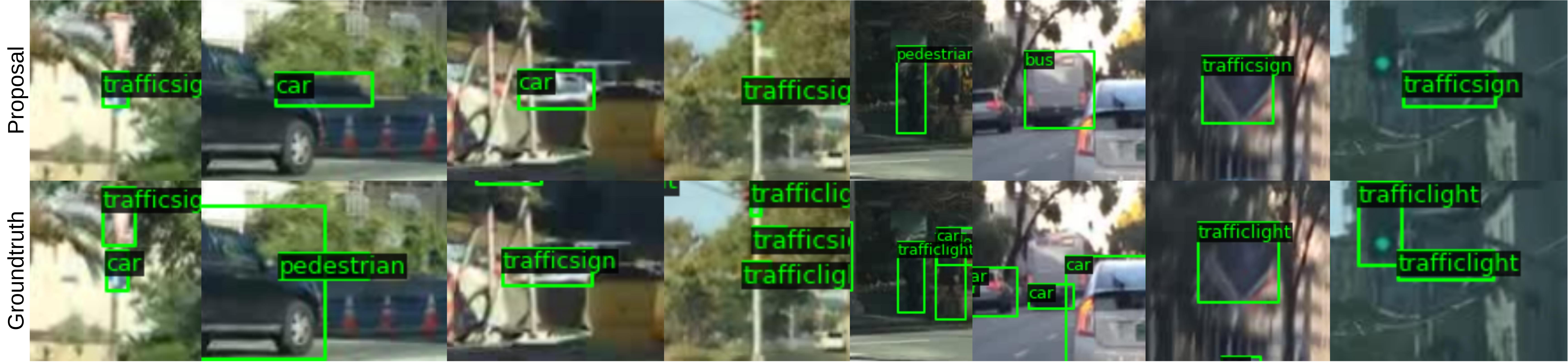}
    }
    \subcaption{BDD}
    \end{subfigure}
    
    \begin{subfigure}[c]{0.995\textwidth}
    \resizebox{\linewidth}{!}{
    \includegraphics[]{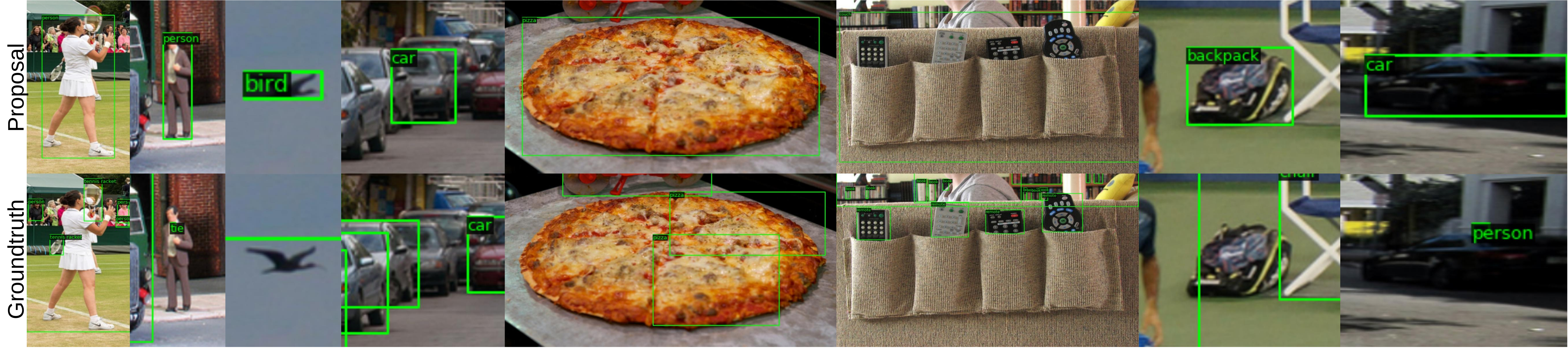}
    }
    \subcaption{COCO}
    \end{subfigure}
    
    \begin{subfigure}[c]{0.995\textwidth}
    \resizebox{\linewidth}{!}{
    \includegraphics[]{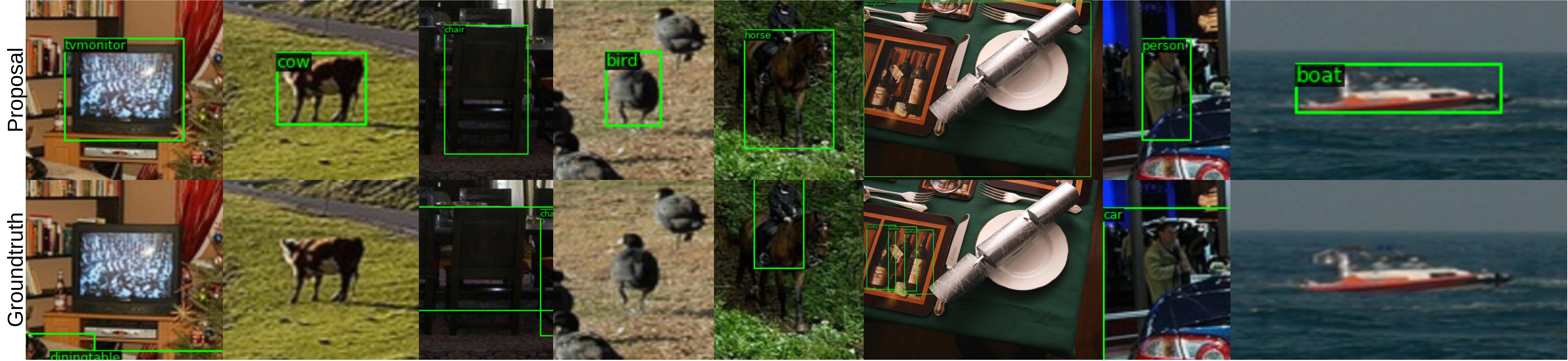}
    }
    \subcaption{VOC}
    \end{subfigure}
    \caption{Visualization of further detected real label errors in test datasets for BDD (top), COCO (center) and VOC (below).}
    \label{fig: label-error-collage-appendix}
\end{figure*}

\paragraph{Theoretical Justification of our Instance-Wise Loss Method}
\begin{figure}
    \centering
        \resizebox{0.55\textwidth}{!}{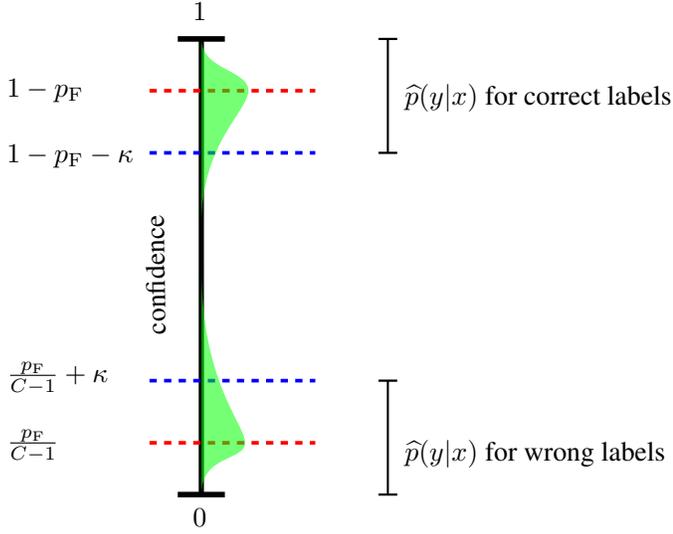}
    \caption{Illustration of the probabilistic statement about predicted confidences conditioned to correct and incorrect given labels.
    PAC learning leads to concentration of the confidences around \(1 - \pf\) and \(\tfrac{\pf}{C - 1}\), respectively.
    The separation on the confidences carries over to the cross entropy loss.}\label{fig:classification-illustration}
\end{figure}
 Our goal is to show that the flip of a test label is statistically captured by the cross entropy loss evaluated at a deep neural network's (DNN\footnote{Technically, it is not required that the model is a DNN as long as PAC-learnability is fulfilled.}) prediction on a test sample \(x\) and the corresponding label \(y\).

The rough intuition for this statement is that a probably approximate correct learner \cite{shalev2014understanding} (PAC-learner) \(\widehat{p}\) has probabilistic bounds for having predictive distribution close to the data-generating distribution \(p\).
Therefore, sufficient data sampling and empirical loss minimization will lead to statistical concentration of confidences \(\widehat{p}\) around \(p\).
If \(p\) does not suffer from too strong label noise, we obtain separation between confidences on incorrect and confidences on correct labels.
This separation then carries over to the negative log-likelihood (i.e., cross entropy) loss by monotony.

 We assume data points \((x, y) \sim p\) following some noisy data generating distribution \(p\), where \(x \sim P_x\) follows a marginal distribution $P_x$.
 In practice, training and test data originate from the same data pool and we do not see any reason to assume that they follow different labeling procedures.
 However, it is sufficient to require that for testing data \((x, y)\), \(x\) follows the same marginal distribution \(x \sim P_x\).
%  Hence, we assume that both training and testing data follow the same distribution \(p\).
Our proof builds on the existence of a true labeling function \(f: x \mapsto y\) and the assumption that the data distribution \(p\) introduces stochastic flips of labels that occur with a fixed uniform rate \(\pf \in [0, 1)\).
This flip probability \(\pf\) uniformly distributed over all \(C - 1\) incorrect classes which are not \(f(x)\) leads to the following constraints on \(p\) when conditioned to \(x\):
\begin{equation}
    p(f(x)|x) := 1 - \pf,
    \quad
    p(y | x) := \pf / (C - 1)
\end{equation}
$\forall y \neq f(x)$.
% \(\) for a true label and \(\) otherwise.
A statistical model \(\widehat{p} = \widehat{p}(y | x)\) PAC-learns classification on samples of the (noisy) data generating distribution \(p = p(y | x)\).
In the present treatment, we assume PAC-learning with respect to the Kullback-Leibler (KL) divergence \(\DKL(p(\cdot | x) \| \widehat{p}(\cdot | x)) = - \int \log \left(\frac{\mathrm{d} \widehat{p}(y | x)}{\mathrm{d} p(y | x)}\right) p(y | x) \, \mathrm{d} y\).
In the following our goal is to show probabilistic statements about the cross entropy loss
\begin{equation}
    \CE(\widehat{p}(x)\|y) := - \sum_{c = 1}^C y_c \cdot \log(\widehat{p}_c(x))
\end{equation}
on test data pairs \((x, y)\).
We show that the loss is above a certain threshold if an incorrect label is given and below some threshold in case of a correct, non-flipped label.
Non-overlapping intervals indicate that the statistical separation between losses given correct and false labels seen in our experiments can be explained theoretically.

We assume PAC-learnability for the proof.
This assumption can be justified via the error decomposition of empirical risk minimization for the KL divergence over the hypothesis space \(\scH\) with training data \(\{(x_1, y_1), \ldots, (x_n, y_n)\}\):
    \begin{align}
        \begin{split}
            &\mathsf{D}(p \| \widehat{p}) :=\mathbb{E}_{x \sim p_x}[\DKL(p(\cdot | x) \| \widehat{p}(\cdot | x))] \\
            &\,\leq \inf_{h \in \scH} \mathsf{D}(p \| h) \\
            &\,+\left(\frac{1}{n} \sum_{j = 1}^n\CE(\widehat{p}(x_j)\|y_j) - \inf_{h \in \scH}\CE(h(x_j) \| y_j) \right) \\
            &\,+ 2 \cdot \sup_{h \in \scH} \left| \mathsf{D}(p \| h) - \frac{1}{n} \sum_{j = 1}^n \CE(h(x_j)\| y_j) - H(p(\cdot | x)) \right| \\
            &\,< \varepsilon
        \end{split}
    \end{align}
    where \(H = - \sum_{c = 1}^C p(c | x) \cdot \log(p(c | x))\) is the entropy of the data generating distribution\footnote{Together with the cross entropy \(\CE\), the entropy \(H\) yields an unbiased risk function for \(\DKL\).}.
    The first term is the model misspecification error given by \(\scH\).
    In practice, we assume an expressive DNN with a large amount of capacity (appealing to universal approximation) which allows for this error to be negligible.
    In particular, in this case, no restrictions need to be made in the choice of \(\scH\).
    The second term measures the error of the learning algorithm w.r.t.\ an empirical risk minimizer \(h\).
    Similarly to the  term, an expressive DNN trained to convergence leads to small contributions by this term.
    Lastly, the third term is the sampling error made as compared to the loss \(\mathsf{D}(p \| h)\) in the true distribution.
    The third term can be controlled by application of concentration inequalities and chaining under certain assumptions (see \cite{van2014probability}) which is why the sum of the three terms can be made smaller than some fixed \(\varepsilon > 0\) given sufficient amount of data.

\setcounter{prop}{0}
\begin{prop}[Statistical Separation of the Cross Entropy Loss]
    Let training and testing labels be given under a stochastic flip in \(p(\cdot | x)\) with probability \(p_\mathrm{F}\) as above, let the label distribution \(p(\cdot|x)\) be PAC-learnable by the hypothesis space of \(\widehat{p}(\cdot | x)\) w.r.t.\ \(\DKL\) (to precision \(\varepsilon\) and confidence \(1 - \delta\)) and let \(\kappa > 0\).
    If \(\pf < \tfrac{C - 1}{C} (1 - 2\kappa)\), we obtain strict separation of the loss function
    % any drawn sample \((x, y)\) involving a true label \(y\) will have a smaller loss than a sample with a false label 
    \begin{align}
        \begin{split}
            \CE(\widehat{p}(x) \| f(x)) &< - \log(1 - \pf - \kappa) \\ 
            &< - \log(\kappa + \tfrac{\pf}{C - 1}) < \CE(\widehat{p}(x) \| \widetilde{y})
        \end{split}
    \end{align}
    % Then, we have for the cross entropy loss
    % \begin{itemize}
    %     \item in the case of the correct label \(y\): \(\CE(\widehat{p}(x) \| y) \leq - \log(1 - \pf - \kappa)\).
    %     \item in the case of an incorrect label \(y\): \(\CE(\widehat{p}(x) \| y) > - \log(\kappa + \tfrac{\pf}{C - 1})\).
    % \end{itemize}
    for any incorrect label \(\widetilde{y} \neq f(x)\) with probability \(1 - \delta\) over chosen training data and with probability \(1 - \tfrac{2 \varepsilon}{\kappa^2}\) over the choice of \(x\).
\end{prop}

\begin{proof}
    We aim at bounding \(\max_{y = 1, \ldots, C} |p(y | x) - \widehat{p}(y | x)|\) by the total variation distance.
    PAC-learnability asserts that given enough data, the \(\widehat{p}\)-distributions illustrated in Fig.~\ref{fig:classification-illustration} are concentrated around \(1 - \pf\) for true labels and \(\tfrac{\pf}{C - 1}\) for incorrect labels.
    In particular, PAC-learnability implies
    \begin{equation}
        \mathbb{E}_{x \sim p_x} [\DKL(p(\cdot | x) \| \widehat{p}(\cdot | x))] < \varepsilon
    \end{equation}
    with probability \(1 - \delta\) over the choice of training data.
    Let \(\kappa > 0\).
    From this PAC result, we derive bounds for the probability of \(\max_{y = 1, \ldots, C} |p(y | x) - \widehat{p}(y | x)|\) exceeding \(\kappa\) via the total variation distance.
    We have
    \begin{align}
        \begin{split}
            P_x&(\| p(\cdot | x) - \widehat{p}(\cdot | x)\|_{\mathrm{TV}} \geq \kappa) \\
        \leq& \, P_x(\sqrt{2 \DKL(p(\cdot | x) \| \widehat{p}(\cdot | x))} \geq \kappa) \\
        \leq& \, P_x\left(\DKL(p(\cdot | x) \| \widehat{p}(\cdot | x)) \geq \frac{\kappa^2}{2}\right) \\
        \leq& \, \frac{2}{\kappa^2} \mathbb{E}_{x \sim p_x} [\DKL(p(\cdot | x) \| \widehat{p}(\cdot | x))] < \frac{2 \varepsilon}{\kappa^2}
        \end{split}
    \end{align}
    with probability \(1 - \delta\) over the choice of training data.
    Here, the first inequality is the application of Pinsker's inequality and the third due to the Markov inequality.
    
    Assume that we are given a correct label \(y\) for \(x\), then with probability \(1 - \delta\) over training data and with probability \(1 - \tfrac{2 \varepsilon}{\kappa^2}\) over sampling \(x\), we have that 
    \begin{align}
        \begin{split}
            |p(y | x) - \widehat{p}(y | x)| =& \,
        |(1 - \pf) - \widehat{p}(y | x)| \\
        \leq&\, \max_{y} |p(y | x) - \widehat{p}(y | x)| \\
        \leq&\, \|p(\cdot | x) - \widehat{p}(\cdot | x) \|_{\mathrm{TV}} < \kappa.
        \end{split}
        % \implies \widehat{p}(y | x) > 1 - \pf - \kappa,
    \end{align}
    This implies \(\widehat{p}(y | x) > 1 - \pf - \kappa\)
    and therefore, by monotony of the logarithm \(\CE(\widehat{p}(y|x) \| y) < - \log(1 - \pf - \kappa)\).
    Similarly, if \(y\) is any incorrect label, we have the probabilistic statement
    \begin{align}
    \begin{split}
        |p(y | x) - \widehat{p}(y | x)| &=
        \left| \frac{\pf}{C - 1} - \widehat{p}(y | x) \right| \\
        \leq&\, \max_{y} |p(y | x) - \widehat{p}(y | x)| < \kappa,
    \end{split}
        % \implies \widehat{p}(y | x) < \kappa + \frac{\pf}{C - 1},
    \end{align}
    i.e., \(\widehat{p}(y | x) < \kappa + \frac{\pf}{C - 1}\) and we have \(\CE(\widehat{p}(x)\|y) > - \log\left(\kappa + \frac{\pf}{C - 1}\right)\).
    Finally, we obtain separability of losses with true versus false labels in probability if
    \begin{equation}
        1 - \pf - \kappa > \kappa + \frac{\pf}{C - 1} 
        \iff \pf < \frac{C - 1}{C} (1 - 2 \kappa).
    \end{equation}
\end{proof}

\end{document}